\documentclass[journal]{IEEEtran}

\usepackage{graphicx}
\usepackage{mathptmx}      
\usepackage{latexsym}
\usepackage{times}
\usepackage{epsfig}
\usepackage{graphicx}
\usepackage{amsmath}
\usepackage{amssymb}
\usepackage{indentfirst}

\usepackage{bbold}
\usepackage{dsfont}

\usepackage{enumitem}
\usepackage[font=small,labelfont=bf]{caption}

\usepackage{xcolor}
\usepackage{amsfonts}

\usepackage{caption}
\usepackage{subcaption}
\usepackage{bm}
\usepackage{isomath}

\usepackage[numbers]{natbib}

\usepackage[pagebackref=true,breaklinks=true,letterpaper=true,colorlinks,bookmarks=false]{hyperref}
\usepackage{footmisc}
\usepackage{stmaryrd}

\usepackage{fixltx2e}
\usepackage{dblfloatfix}
\usepackage{pbox}
\usepackage{capt-of}

\usepackage[normalem]{ulem}
\usepackage{multirow}
\usepackage{colortbl}

\usepackage{mathtools}
\usepackage{array}

\usepackage{amsmath}
\usepackage{amssymb}

\usepackage{amsthm}
\usepackage{indentfirst}
\usepackage{colortbl}
\usepackage{tabularx}
\usepackage{arydshln}
\usepackage{xspace}
\usepackage{booktabs}

\renewcommand\vec[1]{\ensuremath\boldsymbol{#1}}
\renewcommand\cdots{...}

\newcommand{\tF}{\vec{\mathcal{F}}}

\newcommand{\mY}{\mathbf{Y}}

\newcommand{\mX}{\mathbf{X}}
\newcommand{\vx}{\mathbf{x}}

\newcommand{\mbr}[1]{\mathbb{R}^{#1}}

\newcommand{\idx}[1]{\mathcal{I}_{#1}}

\newcommand{\tR}{\vec{\mathcal{R}}}

\newcommand{\vphi}{\boldsymbol{\phi}}
\newcommand{\vpsi}{\boldsymbol{\psi}}

\newcommand{\mPsi}{\vec{\Psi}}

\newcommand{\mF}{\vec{F}}

\newcommand{\set}[1]{\left\{#1\right\}}

\DeclareMathOperator*{\argmin}{arg\,min}

\DeclareMathOperator*{\expect}{\mathbb{E}}

\DeclareMathOperator*{\trace}{Tr}

\DeclareMathOperator*{\avg}{avg}

\newcommand{\expl}[1]{\text{e}^{#1}}

\newtheorem{proposition}{Proposition}

\def\eg{\emph{e.g.}}

\newcommand{\mygthree}[1]{\boldsymbol{\mathcal{G}}\!\left(\!#1\!\right)}

\newcommand{\tG}{\boldsymbol{\mathcal{G}}}

\newcommand{\vOnes}{\boldsymbol{1}}

\newcommand{\mS}{\vec{S}}
\newcommand{\tNnb}{\mathcal{N}}

\newcommand{\mPhi}{\boldsymbol{\Phi}}

\newcommand{\mM}{\boldsymbol{M}}

\newcommand{\stkout}[1]{{\ifmmode\text{\sout{\ensuremath{#1}}}\else\sout{#1}\fi}}

\DeclareMathOperator*{\arcsinh}{arcsinh}
\newcommand{\comment}[1]{}
\newcommand*{\barbar}[1]{\bar{\bar{#1}}}
\newcommand*{\barbarbar}[1]{\bar{\bar{\bar{#1}}}}

\makeatletter
\@namedef{ver@everyshi.sty}{}
\makeatother

\makeatletter
\DeclareRobustCommand\onedot{\futurelet\@let@token\bmv@onedotaux}
\def\bmv@onedotaux{\ifx\@let@token.\else.\null\fi\xspace}
\DeclareRobustCommand\onespace{\futurelet\@let@token\bmv@onespaceaux}
\def\bmv@onespaceaux{\ifx\@let@token \else \null\fi\xspace}

\def\eg{\emph{e.g}\onedot} 
\def\ie{\emph{i.e}\onedot} 
\def\cf{\emph{c.f}\onedot} 
\def\etc{\emph{etc}\onedot} 
\def\wrt{w.r.t\onedot} 
\def\etal{\emph{et al}\onedot}
\makeatother

\def\PK#1{{\color{red}{{[#1]}}}}

\begin{document}
\title{Multi-level Second-order Few-shot Learning}
%
%
%

\def\PK#1{\onespace[{\it \color{red} PK: {#1}}].}
\def\hg{}

\author{Hongguang~Zhang,~\IEEEmembership{}
        Hongdong~Li,~\IEEEmembership{}
        and~Piotr~Koniusz~\IEEEmembership{}
\thanks{H. Zhang* is with Systems Engineering Institute, AMS. H. Li and P. Koniusz are with the College of Engineering and Computer Science, Australian National University. P. Koniusz is also with Data61/CSIRO. E-mail: \url{zhang.hongguang@outlook.com}, \url{hongdong.li@anu.edu.au}, \url{piotr.koniusz@data61.csiro.au}. Code: \url{https://github.com/HongguangZhang/mlso-tmm-master}.}
\thanks{This work is supported by National Natural Science Foundation of China (Grant No. 62106282), and Equipment Development Research Fund (Grant No. ZXD2020C2316).}
\thanks{Part of this work wass done during H. Zhang’s
stay at the ANU.  $\quad$Manuscript received May 12, 2021, accepted Jan. 9, 2022. $\quad$DOI: \url{https://doi.org/10.1109/TMM.2022.3142955}\vspace{-0.7cm}}
}

\markboth{IEEE Transactions on Multimedia}%
{Shell \MakeLowercase{\textit{et al.}}: Bare Demo of IEEEtran.cls for IEEE Journals}
%

\maketitle

\begin{abstract}
We propose a  Multi-level Second-order (MlSo) few-shot learning network for supervised or unsupervised few-shot image classification and few-shot action recognition. We leverage so-called power-normalized second-order base learner streams combined with features that express multiple levels of visual abstraction, and we use self-supervised discriminating mechanisms. {As Second-order Pooling (SoP) is popular in image recognition, we employ its basic element-wise variant in our pipeline. The goal of multi-level feature design is to extract feature representations at different layer-wise levels of CNN, realizing several levels of visual abstraction to achieve robust few-shot learning. As SoP can handle convolutional feature maps of varying spatial sizes, we also introduce image inputs at multiple spatial scales into MlSo. To exploit the discriminative information from multi-level and multi-scale features, we develop a Feature Matching (FM) module that reweights their respective branches. We also introduce a self-supervised step, which is a discriminator of the spatial level and the scale of abstraction.
Our pipeline is trained in an end-to-end manner. With a simple architecture, we demonstrate respectable results on standard datasets such as Omniglot, \textit{mini--}ImageNet, \textit{tiered--}ImageNet, Open MIC, fine-grained datasets such as CUB Birds, Stanford Dogs and Cars, and action recognition datasets such as HMDB51, UCF101, and \textit{mini--}MIT.}
\end{abstract}


%
\IEEEpeerreviewmaketitle

\vspace{-0.15cm}
\section{Introduction}
\label{sec:intro}
Convolutional Neural Networks (CNNs) have advanced a variety of models \eg,  object category recognition,  scene classification and fine-grained image recognition. However, CNNs rely on large numbers of training labeled images and cannot be easily adapted to new tasks given very few samples. 
In contrast, the ability of humans to  learn new visual concepts from very few examples highlights the superiority of biological vision. Thus, researchers study the so-called few-shot learning paradigm for which networks are trained or adapted to new concepts with few training samples. For  example, recent few-shot learning approaches  \cite{vinyals2016matching,snell2017prototypical,sung2017learning,NIPS2017_7082} build on the notion of similarity learning \cite{metric_old,kissme,Mehrtash_CVPR_2018}. In this paper, we study the one- and few-shot learning problems, and we focus on a simple design capturing robust statistics for the purpose of similarity learning.

In what follows, we employ second-order statistics of datapoints, which have advanced the performance of numerous methods, including object recognition, texture categorization, action representation, and tracking  \cite{tuzel_rc,porikli2006tracker,wang2011tracking,guo2013action,carreira_secord,me_tensor_tech_rep}. For example, in the popular region covariance descriptors~\cite{tuzel_rc}, a covariance matrix computed over multimodal features from image regions is used as an object representation for recognition and tracking. Covariance descriptors have been extended to many other applications~\cite{tuzel_rc,porikli2006tracker,wang2011tracking,guo2013action} including end-to-end training of CNNs, leading to state-of-art results on action recognition, texture classification, scene and fine-grained recognition \cite{koniusz2016tensor,koniusz2017domain,koniusz2018deeper}. 
As second-order representations capture correlation patterns of features, they are a powerful tool used in several recognition pipelines \cite{face_cooc,deep_cooc,bilinear_finegrained,koniusz2017higher,koniusz2018deeper,lin2018o2dp,saimunur_redro,maxexp,hosvd,sice}. 


A typical few-shot learning network consists of a backbone that generates image features, and a base learner that learns to classify the so-called query images (\cf class labels). In this paper, we use a multi-level network to obtain multiple levels of feature abstraction based on second-order features. We leverage intermediate outputs from the backbone, which helps the pipeline capture relations between the query and support images at multiple levels of abstraction. We note that a cascaded network was used before by GoogLeNet \cite{szegedy_google} with the goal of image classification rather than the similarity learning in the few-shot regime, which is a novel learning scenario not explored before.

By analyzing the class-wise activation maps, we ascertain that the features extracted from different levels of the backbone generally describe objects with respect to their different visual properties. Thus, such complementary to each other activation maps improve modeling of object relations across different levels of abstraction. To this end, we leverage second-order statistics formed from features of multi-level network streams. Firstly, we form and pass such second-order representations via the so-called Power Normalization (PN) to prevent the so-called burstiness effect \cite{koniusz2018deeper}, a statistical uncertainty of feature counts. {As second-order pooling can effectively process feature maps of different spatial resolutions, we also employ inputs at multiple spatial scales to improve the quality of matching between objects at various scales. Secondly, we apply a so-called Feature Matching module which determines the importance of each level of abstraction  and scale per query-support pair. Subsequently, we form relationship descriptors and pass them to a so-called base learner which learns the similarity by comparing relationship descriptors (representing query-support pairs) via the Mean Square Error (MSE) loss}. Finally, we refine the multi-level second-order matrices, each corresponding to some level of abstraction in the multi-level network, by applying a  self-supervised pretext task \cite{gidaris2019boosting,he2020momentum,chen2020simple} with the level and scale indexes used as auxiliary labels. Such a self-supervised step  helps the multi-level network learn more distinctive and complementary abstraction responses. 

We apply our network to the few-shot image and action recognition tasks. In contrast to the large-scale object classification, few-shot learning  requires an investigation into the effective use of multiple levels of feature abstraction, combined with second-order relationship descriptors, to determine the best performing architecture. As second-order statistics require appropriate pooling for such a new problem, we employ PN which is known to act as detector of visual features. PN discards the so-called nuisance variability (burstiness), the uncertainty of the frequency of specific visual features which vary unpredictably from image to image of the same class \cite{koniusz2018deeper}. We speculate that, as we capture relationships between multiple images in a so-called episode, such a nuisance variability would be multiplicative \wrt the number of images per episode. PN limits such a harmful effect. 
%
%

\vspace{0.25cm}
Our contributions are summarized below:
\vspace{0.1cm}
\renewcommand{\labelenumi}{\roman{enumi}.}
\begin{enumerate}[leftmargin=0.5cm]
\item {We propose to generate scale-wise second-order representations at multiple levels of abstraction via a multi-level network, and we introduce the Feature Matching (FM) module to reweight the importance of each abstraction level and scale. FM selects the most discriminative pairs for relation learning. We show the importance of reweighting and matching across multiple spatial scales at multiple levels of feature abstraction.}
\vspace{0.15cm}
\item We investigate how to build second-order relational descriptors from feature maps to capture the similarity between query-support pairs for few-shot learning.
\vspace{0.15cm}
\item {We develop a self-supervised discriminator acting on scale-wise second-order representations of multiple levels of abstraction whose role is to predict the index of abstraction and scale, thus improving the complementarity and distinctiveness between different abstraction levels and scales}.
\end{enumerate}

\vspace{0.15cm}
To the best of our knowledge, we are the first to investigate multiple levels of feature abstraction, multiple input scales, and self-supervision for one- and few-shot learning. In this work, we build upon the SoSN model \cite{sosn}, the first few-shot learning model successfully leveraging second-order pooling. 
\section{Related Work}
\label{sec:related}

Below, we describe popular  one- and few-shot learning models, and discuss other related topics.

\vspace{-0.2cm}
\subsection{Learning From Few Samples}
\label{sec:related_few_shot}

The ability of {\em`learning quickly from only a few examples is definitely the desired characteristic to emulate in any brain-like system'} \cite{book_nip}. This desired principle poses a challenge to CNNs which typically leverage large-scale datasets \cite{ILSVRC15}. Current trends in computer vision highlight the need for the {\em `ability of a system to recognize and apply knowledge and skills learned in previous tasks to novel tasks or new domains, which share some commonality'}. 
For one- and few-shot learning, a robust `{\em transfer of particle}', introduced in 1901 by Woodworth \cite{woodworth_particle}, is also a desired mechanism because generalizing based on one or few datapoints to account for intra-class variability of thousands images is formidable.

\vspace{0.05cm}
\noindent{\textbf{One- and Few-shot Learning (FSL)}} have been studied in computer vision in both shallow \cite{NIPS2004_2576,BartU05,fei2006one,lake_oneshot} and deep learning scenarios \cite{koch2015siamese,vinyals2016matching,snell2017prototypical,finn2017model,sung2017learning,zhu2020attribute,huang2020low,hu2020learning,phaphuangwittayakul2021fast,r1-1,r1-2,KTN,TAML,FLAT}. 
%
For brevity, we review only the deep learning techniques. Siamese Network \cite{koch2015siamese}, a CNN based on two-streams, generates image descriptors and learns the similarity between them. 
Matching Network \cite{vinyals2016matching} introduces the concept of the support set, and the $L$-way $Z$-shot learning protocols. Matching Network  captures the similarity between a query and several support images, and generalizes well to previously unseen  test classes. 
Prototypical Networks \cite{snell2017prototypical} learn a model that computes distances between a datapoint and prototype representations of each class. 
Model-Agnostic Meta-Learning (MAML) \cite{finn2017model} and Task-Agnostic Meta-Learning (TAML) 
perform a rapid adaptation to new tasks via meta-learning. 
Moreover, a large family of meta-learning approaches apply some form of the gradient correction \eg, 
Meta-SGD \cite{Li2017MetaSgd}, MAML++ \cite{Antoniou2019Htmaml}, Reptile \cite{Nichol2018Reptile}, CAVIA \cite{Zintgraf2019Cavia} LEO \cite{Rusu2019LEO} and ModGrad~\cite{Christian2020ModGrad} adapt the step-size of the gradient updates. 
Relation Net \cite{sung2017learning}  learns the relationship between query and support images by leveraging a similarity learning network wired with  the backbone. Note that relationship learning in few-shot learning is closely related to metric learning \cite{metric_old,kissme} rather than relationship learning in graphs \cite{tang2009relational,Wang_2020}. 
SalNet \cite{zhang2019few} is an efficient saliency-guided end-to-end meta-hallucination approach. 
{AFL \cite{zhu2020attribute} proposes a novel attribute-guided two-layer  learning framework to improve the generalized performance of image representations. LRPABN \cite{huang2020low} uses an effective low-rank pairwise bilinear pooling operation to capture the nuanced differences between images. FAML \cite{phaphuangwittayakul2021fast} proposes a novel GAN-based few-shot image generation approach, which is capable of generating new realistic images for unseen target classes in the low-sample regime. 
Zhu \etal \cite{zhu2020one} propose a novel global grouping metric to incorporate the global context, resulting in a per-channel modulation of local relation features. Moreover, a cross-modal retrieval can be performed by a modified meta-learning framework \cite{meta_retrieval}.} Finally, Graph Neural Networks (GNN) \cite{kipf2017semi,felix2019icml,johannes2019iclr,uai_ke,hao2021iclr} have also been used in  few-shot learning  \cite{gnn,Kim_2019_CVPR,Gidaris_2019_CVPR,wang20213d}, and achieved competitive results. 
{Self-supervision has also been studied in few-shot learning \cite{gidaris2019boosting,su2020does,FLAT,arl} where transformation-based auxiliary self-supervised classifiers are employed to improve the robustness  of few-shot learning models. For instance, FLAT \cite{FLAT} uses a self-supervision strategy to  pre-train the auto-encoder via the reconstruction of transformations applied to images, followed by training with the supervised loss.
 }

Our pipeline is somewhat similar to Relation Net \cite{sung2017learning} and Prototypical Networks \cite{snell2017prototypical} in that we use the two basic building blocks underlying such approaches, that is, the feature encoder (or backbone) and the similarity learning network. However, Relation Net and Prototypical Net are first-order models which do not use multiple levels of feature abstraction. In contrast, we investigate second-order representations with PN to capture correlations of features. We also use multiple intermediate feature outputs to obtain different levels of feature abstraction. Our work builds on our SoSN model \cite{sosn},  which we extend in this work by adding multiple levels of feature abstraction, multi-scale inputs, a feature matching mechanism, and a self-supervision step. Finally, we also include an extension of our pipeline to the problem of few-shot action recognition. 

\vspace{0.05cm}
\noindent{\textbf{Few-shot Action Recognition (FSAR)}} has been studied by the limited number of recent works \cite{mishra2018generative,guo2018neural,xu2018dense,cmn,arn,wang20213d}. Though action recognition \cite{r1-3,r1-4,r1-5} has been studied for a long time, it remains a challenging problem under the few-shot setting. 
Mishra \etal \cite{mishra2018generative} propose a generative framework for zero- and few-shot action recognition, by modeling each action class by a probability distribution. Guo \etal \cite{guo2018neural} leverage neural graph matching to learn to recognize previously unseen 3D action classes. Xu \etal \cite{xu2018dense} propose a dilated network to simultaneously capture local and long-term spatial temporal information. Zhu \etal \cite{cmn} propose a novel compound memory network. {Hu \etal \cite{hu2020learning} learn a dual-pooling GNN to improve the discriminative ability for selecting the representative video content and refine video relations. Finally, noteworthy are few-shot  pipelines such as VideoPuzzle \cite{VideoPuzzle} which generates aesthetically enhanced long-shot videos from short video clips, and JEANIE \cite{wang20213d} which performs FSAR on datasets of articulated human 3D body joints.}


\subsection{Second-order Statistics/Power Normalization}
\label{sec:related_pn}

Below we discuss several shallow and CNN-based methods which use second-order statistics. We conclude with details of so-called pooling and Power Normalization.

\vspace{0.05cm}
\noindent{\textbf{Second-order statistics }}
have been used for texture recognition \cite{tuzel_rc, elbcm_brod} by so-called Region Covariance Descriptors (RCD) 
and further  applied to tracking \cite{porikli2006tracker}, semantic segmentation \cite{carreira_secord} and object category recognition \cite{me_tensor_tech_rep,koniusz2017higher}. 
Co-occurrence patterns have also been used in the CNN setting. 
A recent approach \cite{deep_cooc} extracts feature vectors at two separate locations in feature maps to perform an outer product in a CNN co-occurrence layer. Higher-order statistics have also been used for action recognition from the body skeleton sequences \cite{koniusz2016tensor,hosvd} and for domain adaptation \cite{koniusz2017domain}. 
%
In this work, we perform end-to-end training for one- and few-shot learning by the use of second-order relation descriptors (a novel proposition) that capture relations between the query and support images (or videos) before passing them to the similarity learning network. 
%
%
Second-order statistics have to deal with the so-called burstiness,  `{\em the property that a given visual element appears more times in an image than a statistically independent model would predict}' \cite{jegou_bursts}. This is achieved by Power Normalization~\cite{me_ATN,me_tensor_tech_rep} which is known to suppress this burstiness. 
PN has been extensively studied and evaluated in the context of Bag-of-Words \cite{me_ATN,me_tensor_tech_rep,koniusz2017higher} and category recognition with deep learning \cite{koniusz2018deeper,maxexp}. 

A theoretical relation between Average and Max-pooling was studied in~\cite{boureau_midlevel}, which highlighted the underlying statistical reasons for the superior performance of max-pooling. 
A survey \cite{me_ATN} showed that so-called MaxExp  pooling in~\cite{boureau_pooling} acts as a detector of `\emph{at least one particular visual word being present in an image}', and thus it can be approximated with a simple $\text{min}_n(1, \eta\phi_{kn})$ for $\eta\!>\!0$, whose variant we use in this paper.  

\section{Background}
Below we detail our notations, and explain how to compute second-order statistics with 
PN. 

\vspace{0.05cm}
\noindent\textbf{Notations.} 
Let $\vx\in\mbr{d'}$ be a $d'$-dimensional feature vector. $\idx{N}$ stands for the index set $\set{1, 2,\cdots,N}$. 
We also define $\vOnes\!=\![1,\cdots,1]^T$. Operators $;_1$, $;_2$ and $;_3$ denote concatenation of tensors along the first, second and third mode, respectively. Operator $(:)$ denotes vectorization of a matrix or tensor. Typically, capitalized bold symbols such as $\mPhi$ denote matrices, lowercase bold symbols such as $\vphi$ denote vectors, and regular symbols such as $n$ or $Z$ denote scalars. Also, $\Phi_{ij}$ is the $(i,j)$-th coefficient of $\mPhi$.

\vspace{0.05cm}
\noindent\textbf{Autocorrelation matrices.} 
\label{sec:som}
The linearization of sum of Polynomial kernels results in two  autocorrelation matrices.
\begin{proposition}
\label{pr:linearize}
Let $\mPhi_A\equiv\{\vphi_n\}_{n\in\tNnb_{\!A}}$, $\mPhi_B\!\equiv\{\vphi^*_n\}_{n\in\tNnb_{\!B}}$ be datapoints from two images $\Pi_A$ and $\Pi_B$, where $N\!=\!|\tNnb_{\!A}|$ and $N^*\!\!=\!|\tNnb_{\!B}|$ are the numbers of data vectors \eg, obtained from the last convolutional feature map of CNN for images $\Pi_A$ and $\Pi_B$. Autocorrelation matrices emerge from the linearization of the sum of Polynomial kernels of degree $2$:
\vspace{-0.1cm}
\begin{align}
& K(\mPhi_A, \mPhi_B)\!=\!\left<\mF(\mPhi_A),\mF(\mPhi_B)\right>\!=\!\frac{1}{NN^*\!}\!\!\sum\limits_{
n\in \tNnb_{\!A}}\sum\limits_{n'\!\in \tNnb_{\!B}\!}\left<\vphi_n, \vphi^*_{n'}\right>^2\!, \nonumber \\
&\quad \text{ where }\;\;\mF(\mPhi)\!=\!\frac{1}{N}\sum\limits_{n\in \tNnb}\vphi_n\vphi_n^T.\!\! 
\label{eq:hok1}
\end{align}
\end{proposition}
\begin{proof}
See  \cite{me_tensor_tech_rep,koniusz2017higher} for a derivation of this expansion.
\end{proof}


\noindent\textbf{Power Normalization.} In what follows, we use 
%
\comment{
\begin{align}
& \!\!\!\!\!\!\!\frac{1}{NN^*\!}\!\!\sum\limits_{
n\in \tNnb_{\!A}}\sum\limits_{n'\!\in \tNnb_{\!B}}\!\!\!\left<\vphi_n, \vphi^*_{n'}\!\right>^2\!\!=\!
\Big\langle\frac{1}{N}\sum\limits_{
n\in \tNnb_{\!A}}{\vphi_n\vphi_n^T}, \frac{1}{N^*\!}\sum\limits_{
n\in \tNnb_{\!B}}{\vphi^*_{n'}{\vphi^*_{n'}}^{\!\!\!T}}\Big\rangle.\!\!\label{eq:hok2}
\end{align}
}
autocorrelation matrices 
%
%
with a pooling operator 
\comment{
will be replaced by various Power Normalization functions. This yields a (kernel) feature map{\color{red}\footnotemark[3]}:
\begin{align}
& \!\!\!\mPsi\left(\{\vphi_n\}_{n\in\tNnb}\right)=\tG\Big(\frac{1}{N}\sum_{n\in\mathcal{N}}\!\vphi_n\vphi_n^T\Big)=\tG\Big(\frac{1}{N}\mPhi\mPhi^T\Big).\label{eq:hok3}
\end{align}
\vspace{-0.5cm}
}
%
%
%
related to the following two propositions.

\begin{proposition}
\label{pr:cooc}
Assume two event vectors $\vphi,\vphi'\!\!\in\!\{0,1\}^{N}$ which store the $N$ trials each, performed according to the Bernoulli distribution under the i.i.d. assumption. 
Let $p$ be the probability of an event $(\phi_{n}\!\wedge\!\phi'_{n}\!=\!1)$ denoting a co-occurrence, and $1\!-\!p$ of $(\phi_{n}\!\wedge\!\phi'_{n}\!=\!0)$ denoting the lack of co-occurrence. Let  $p\!=\!\avg_n\phi_n\phi'_{n}$ denote an expected value. Then the probability of at least one co-occurrence event $(\phi_{n}\!\wedge\!\phi'_{n}\!=\!1)$ in $\phi_n$ and $\phi'_n$ simultaneously in $N$ trials becomes $\psi\!=\!1\!-\!(1\!-\!p)^{\eta}$ for $\eta\!=\!N$.
\end{proposition}
\begin{proof}
See \cite{koniusz2018deeper} for a proof.
\end{proof}
\begin{proposition}
\label{pr:axmin}
The first-order Taylor expansion of $1\!-\!(1\!-\!p)^{\eta}$ around $p\!=\!0$ and $p\!=\!1$ equals $\eta p$ and $1$, respectively. Thus, we have $1\!-\!(1\!-\!p)^{\eta}\leq\text{min}(\eta p, 1)$ on $p\!\in\![0;1]$. If we treat coefficients $(i,j)$ of matrix $\mM\!=\!{\mPhi\mPhi^T}\!/{\trace(\mPhi\mPhi^T)}$ as approximately proportional to the co-occurrence probability of $\phi_{in}$ and $\phi_{jn}, \forall n\!\in\!\idx{N}$, then for $\mPhi\!\geq\!0$, we obtain PN maps $\mPsi(\mM; \eta)\!=\!\text{min}(\eta\mM, 1)$, where $\eta\!\approx\!N$ is an adjustable parameter accounting for the fact that we do not operate on the actual variable $p$ drawn according to the Bernoulli distribution.
\end{proposition}

\noindent We adopt the above PN operator as it works well in practice according to survey \cite{me_ATN}. Proposition \ref{pr:axmin} is a novel proposition. 


%
\section{Pipeline}
{Below, we describe our network followed by our relationship descriptors whose role is to capture co-occurrences in the image and video representations. We also detail the Feature ENcoder (FEN) with multiple levels of abstraction outputs, the Feature Matching (FM) module and the self-supervised Visual Abstraction Level and Scale Discriminator.} 

\begin{figure*}[t]
	\centering
	\includegraphics[width=\linewidth]{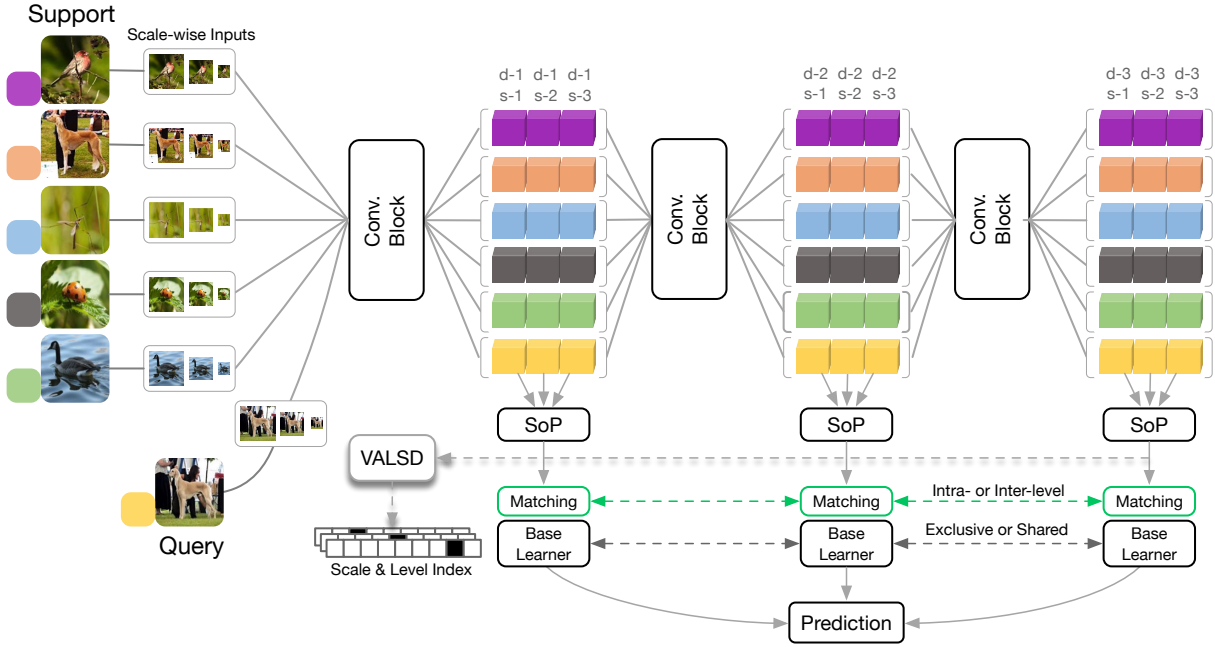}
	\caption{\small The pipeline of our Multi-level Second-order (MlSo) few-shot learning. Convolutional feature maps corresponding to inputs at multiple spatial scales are extracted at multiple levels of the CNN Feature Encoder to perform advanced matching. At each level, we apply second-order pooling with the Matching Module before passing these scale-wise representations to the base learner, with the MSE loss per level of visual abstraction applied, and the Visual Abstraction Level and Scale Discriminator.}
	\label{fig:md_sosn}
\end{figure*}

\begin{figure}[t]
	\centering
	\includegraphics[width=0.8\linewidth]{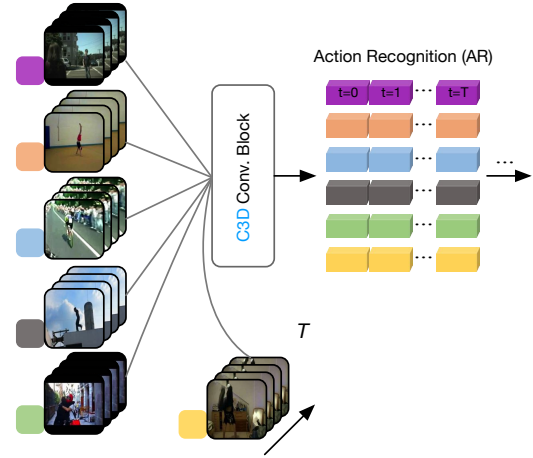}
	\caption{\small For Few-shot Action Recognition, we use the C3D convolutional blocks and the second-order pooling which aggregates over the temporal mode (in addition to spatial locations of feature maps).}
	\label{fig:fsar}
\end{figure}

\begin{figure*}[h]
	\centering
	\includegraphics[width=0.9\linewidth]{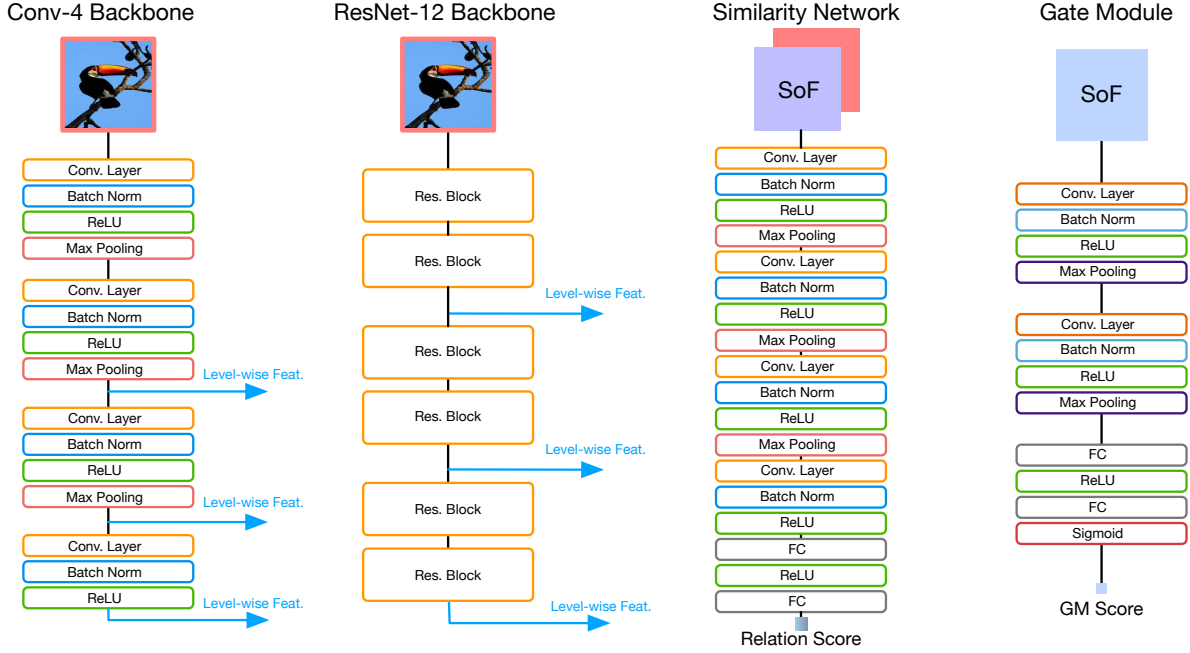}
	\caption{\small The architecture of FEN (Conv-4 and ResNet-12 backbones), SN base learner, Gate Module employed in our approach.}
	\label{fig:blocks}
\end{figure*}

\subsection{Multi-level Second-order Few-shot Learning}

Our Multi-level Second-order (MlSo) few-shot learning pipeline, shown in Figure \ref{fig:md_sosn}, consists of two major parts which are (i) the backbone (feature encoder, FEN) with multiple branches of features and (ii) the base learner \ie, Similarity Network (SN), Logistic Regression (LR) or Nearest Neighbor (NN) classifier. The role of the backbone is to generate convolutional feature vectors which are then used as image descriptors.  The base learner infers relations between query-support pairs. Our work differs from the Relation Net \cite{sung2017learning} in that we use multiple (\ie, three) levels of feature abstraction  passed to second-order representations with PN, which form query-support relation descriptors. We feed such descriptors into SN for comparison of second- rather than the first-order statistics. For clarity, we first describe the Second-order Similarity Network which uses a single FEN output from the last layer. Subsequently, we explain the details of MlSo.

\vspace{0.05cm}
\noindent{\textbf{SoSN}.} 
Let FEN be defined as an operator $f\!:\mbr{W\!\times\!H} \!\times\! \mbr{|\tF|}\!\shortrightarrow\!\mbr{K\!\times\!N}$, where $W$ and $H$ denote the width and height of an input image, $K$ is the length of feature vectors (number of filters), $N\!=\!N_W\!\cdot\!N_H$ is the total number of spatial locations in the last convolutional feature map. For simplicity, we denote an image descriptor by $\mPhi\!\in\!\mbr{K\!\times\!N}$, where $\mPhi\!=\!f(\mX; \tF)$ for an image $\mX\!\in\!\mbr{W\!\times\!H}$, and $\tF$ are learnable parameters of the encoding network. 
The role of SN, denoted by $r\!:\mbr{K'\!}\!\times\!\mbr{|\tR|}\!\shortrightarrow\!\mbr{}$, is to compare two datapoints encoded as some $K'$ dimensional vectorized second-order representations. Typically, we write $s(\vpsi; \tR)$, where $\vpsi\!\in\!\mbr{K'}\!$, whereas $\tR$ are learnable parameters of the similarity network. 
We define a relationship descriptor  $\vartheta\!:\mbr{K\!\times\!N\!\times\!Z}\!\times\! \mbr{K\!\times\!N}\!\shortrightarrow\!K'\!$ which captures some relationship between descriptors built from the $Z$-shot support images and a query image. This relationship is encoded via computing second-order autocorrelation matrices with PN for query and support embeddings, and forming some relation between query-support features \eg, by concatenation, inner-product, subtraction, \etc, as explained later.

\begin{figure*}[t]
    \centering
    \includegraphics[width=\linewidth]{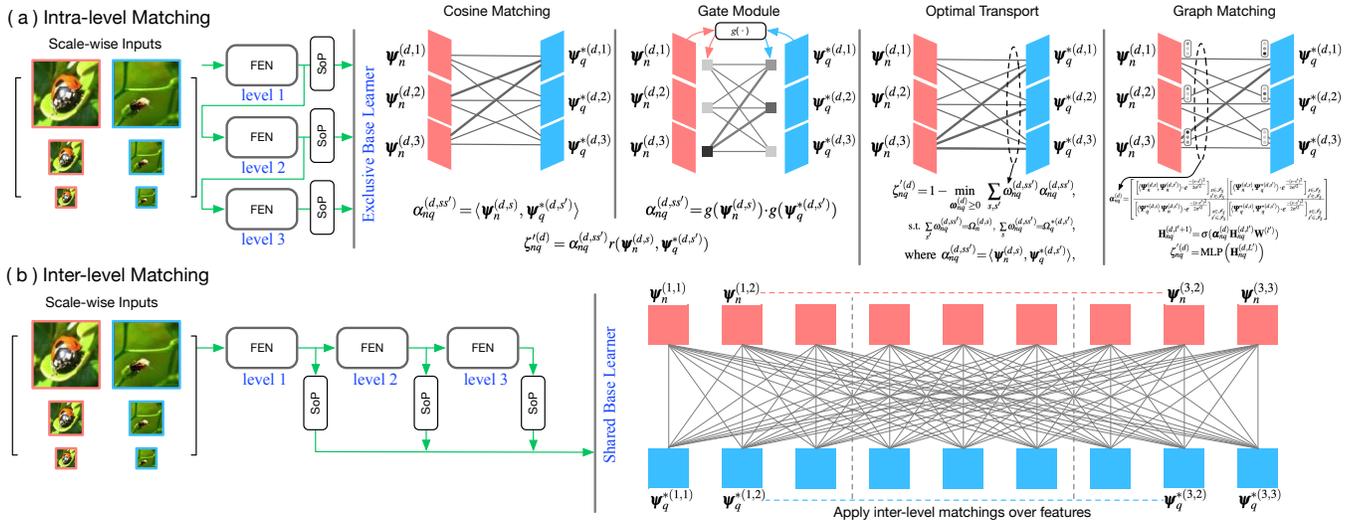}
    \caption{Our matching strategies for multiple spatial scales and multiple levels of feature abstraction. Firstly, we downsample the original support and query images (or videos) to 1/2 or 1/4 of the original resolution. Subsequently, we feed them into the Feature Encoder and construct scale-wise support-query pairs for relation learning. Cosine Matching, Gate Module, Optimal Transport and Graph Matching strategies are detailed in Section \ref{sec:mlrl}.}
    \label{fig:matching}
\end{figure*}

For the $L$-way $Z$-shot problem, assume that we have some support images $\{\mX_n\}_{n\in\mathcal{W}_l}$ from some set $\mathcal{W}_l$ and their corresponding image descriptors $\{\mPhi_n\}_{n\in\mathcal{W}_l}$ form   $Z$-shot relation descriptors.  Moreover, assume that we have one query image $\mX^*\!$ with its image descriptor $\mPhi^*$ (the asterisk usually denotes query-related variables). Both the $Z$-shot and the query embeddings belong to one of $L$ classes in the subset $\mathcal{C}^{\ddag}\!\equiv\!\{c_1,\cdots,c_L\}\!\subset\!\idx{C}\!\equiv\!\mathcal{C}$. Specifically, for an $L$-way problem, one obtains $L$ relation descriptors, where one of $L$ descriptors contains query-support pair of the same class, whereas the remaining $L\!-\!1$ relation descriptors contain non-matching query-support classes.  The $L$-way $Z$-shot learning step can be defined as learning the similarity \wrt relation descriptors:
\begin{align}
& \zeta_{lq}=r\left(\vartheta\!\left([\mPhi_n]_{n\in\mathcal{W}_l},\mPhi_{q}^*\right),\tR\right) \!,\\
&\text{ where }\; \mPhi_n\!=\!f(\mX_n; \tF) \;\text{ and }\; \mPhi^*_q\!=\!f(\mX^*_q\!; \tF),\nonumber
\end{align}
where operator $[\cdot]$ stacks matrices along the third mode, support subsets and query are $(\mathcal{W}_1,\cdots,\mathcal{W}_L,\,q)\in\mathcal{E}$, where $\mathcal{E}$ is a set of episodes. Each $\mathcal{W}_1$ shares the label with $q$, whereas $\mathcal{W}_2,\cdots,\mathcal{W}_L$ do not share the label with $q$.  
%
We use the Mean Square Error (MSE) for the objective of our end-to-end SoSN:
\begin{align}
\argmin\limits_{\tF, \tR} \expect\limits_{(\mathcal{W}_1,\cdots,\mathcal{W}_L,\,q)\sim\mathcal{E}}\sum\limits_{l\in\idx{L}} \left(\zeta_{lq}\! - \delta(c_l\!-\!c^*_q)\right)^2,
\label{eq:sosn1}
\end{align}
\setcounter{equation}{9}
\begin{table*}[t]
\centering
{
\fontsize{8.5}{9}\selectfont
\begin{tabular}[t]{c|c|c|c}
%
Operator &\begin{minipage}{3.5cm}{($\otimes$+F) Full single auto-correlation per support-query concatenated feature vectors.}\end{minipage} & \begin{minipage}{4.5cm}{($\otimes$+R) Rank differing auto-correlations, one per support, one per query, followed by the concatenation.}\end{minipage} & \begin{minipage}{4.5cm}($\otimes$) Auto-correlations, one per support, averaged over shots $\bar{\mPhi}$, one per query, followed by the concatenation.\end{minipage} \\
\hline

{\fontsize{6}{6}\selectfont
\vspace{4cm}
$\vartheta\!\left(\cdot\right)=$
\vspace{-4cm}
}

&{\begin{minipage}{4cm}
\includegraphics[width=4cm]{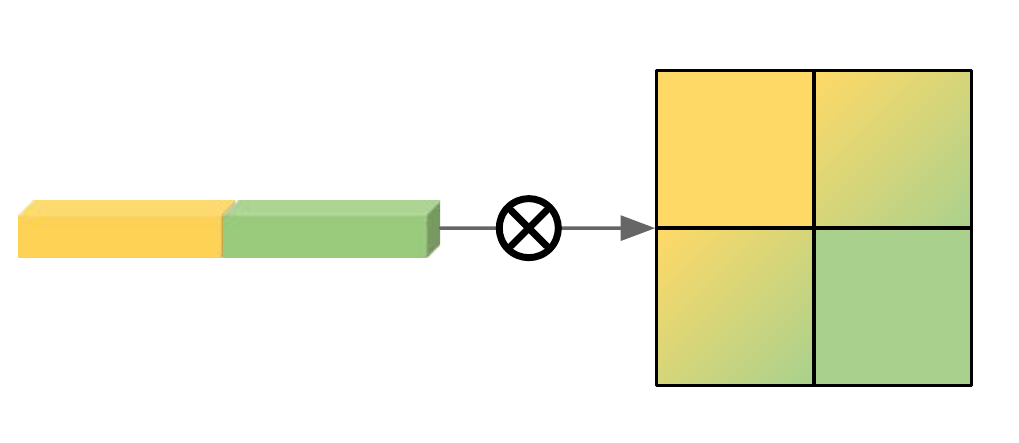}

{\fontsize{6}{6}\selectfont
\begin{equation}
\!\!\!\!\!\!\mPsi\left(\barbarbar{\mPhi}\barbarbar{\mPhi}^T/\trace(\barbarbar{\mPhi}\barbarbar{\mPhi}^T)\right)_{(:)}\nonumber
\end{equation}
\begin{equation}
\text{where }\barbarbar{\mPhi}\!=\!\big[\mPhi_n\,;_1\mPhi^*\big]_{n\in\mathcal{W}_l}\label{eq:concat_f}
\end{equation}
}
\end{minipage}
} &
{\begin{minipage}{4cm}
\includegraphics[width=4cm]{images/operators_2.pdf}

\vspace{-0.2cm}
{\fontsize{6}{6}\selectfont
\begin{equation}
\!\!\!\!\!\!\left[\!\!\!\!\!\!\!\substack{\mPsi\left(\barbar{\mPhi}\barbar{\mPhi}^T\!/\!\trace(\barbar{\mPhi}\barbar{\mPhi}^T)\right);_3\\ \qquad\mPsi\left(\mPhi^*\!\mPhi^{*T}\!/\!\trace(\mPhi^*\!\mPhi^{*T})\right)}\right]_{(:)}
\nonumber
\end{equation}
}
\vspace{-0.2cm}
{\fontsize{6}{6}\selectfont
\begin{equation}\!\!\!\!\!\!\!\!\!\!\text{where }\barbar{\mPhi}\!=\!\big[\mPhi_n\big]_{n\in\mathcal{W}_l}\label{eq:concat_r}
\end{equation}
}
\end{minipage}}
&
{\begin{minipage}{4cm}
\includegraphics[width=4cm]{images/operators_3.pdf}

{\fontsize{6}{6}\selectfont
\begin{equation}
\!\!\!\!\!\!\left[\!\!\!\!\!\!\!\substack{\mPsi\left(\bar{\mPhi}\bar{\mPhi}^T\!/\!\trace(\bar{\mPhi}\bar{\mPhi}^T)\right);_3\\ \qquad\mPsi\left(\mPhi^*\!\mPhi^{*T}\!/\!\trace(\mPhi^*\!\mPhi^{*T})\right)}\right]_{(:)}\nonumber
\end{equation}
\vspace{-0.2cm}
\begin{equation}
\text{where }\bar{\mPhi}\!=\!\!\!\!\sum_{n\in\mathcal{W}_l}\!\!\mPhi_n\label{eq:concat_best}
\end{equation}
}
\end{minipage}}\\
\hline
\end{tabular}
}
\caption{Proposed relationship descriptors $\vartheta$ used in relation learning. Note differences between  equations \eqref{eq:concat_f}, \eqref{eq:concat_r} and \eqref{eq:concat_best}.}\label{tab_methods}
\end{table*}

\vspace{-0.425cm}
\noindent where $c_l$ are labels of support subsets and the query label $c^*_q$ is always set as $c_1$ (so that the query image has the same label as the first support subset),  $\delta(c_l\!-\!c^*_q)$ equals 1 if $c_l\!=\!c^*_q$, 0 otherwise.

For FSAR, we propose a C3D Second-order Similarity Network (C3D SoSN) that is equipped with FEN with C3D convolutional blocks, as in Figure \ref{fig:fsar}. 
Firstly, we obtain C3D-based FEN embeddings over a video and then we compute the autocorrelation matrix with PN to obtain relation descriptors for query-support video pairs. We consider C3D SoSN as a baseline FSAR model. Thus, we do not investigate here any elaborate aggregation strategies, and we do not use the optical flow. However, we believe that further improvements over the  C3D SoSN baseline can be easily achieved.

\begin{figure*}[t]
	\centering
	\includegraphics[width=0.8\linewidth]{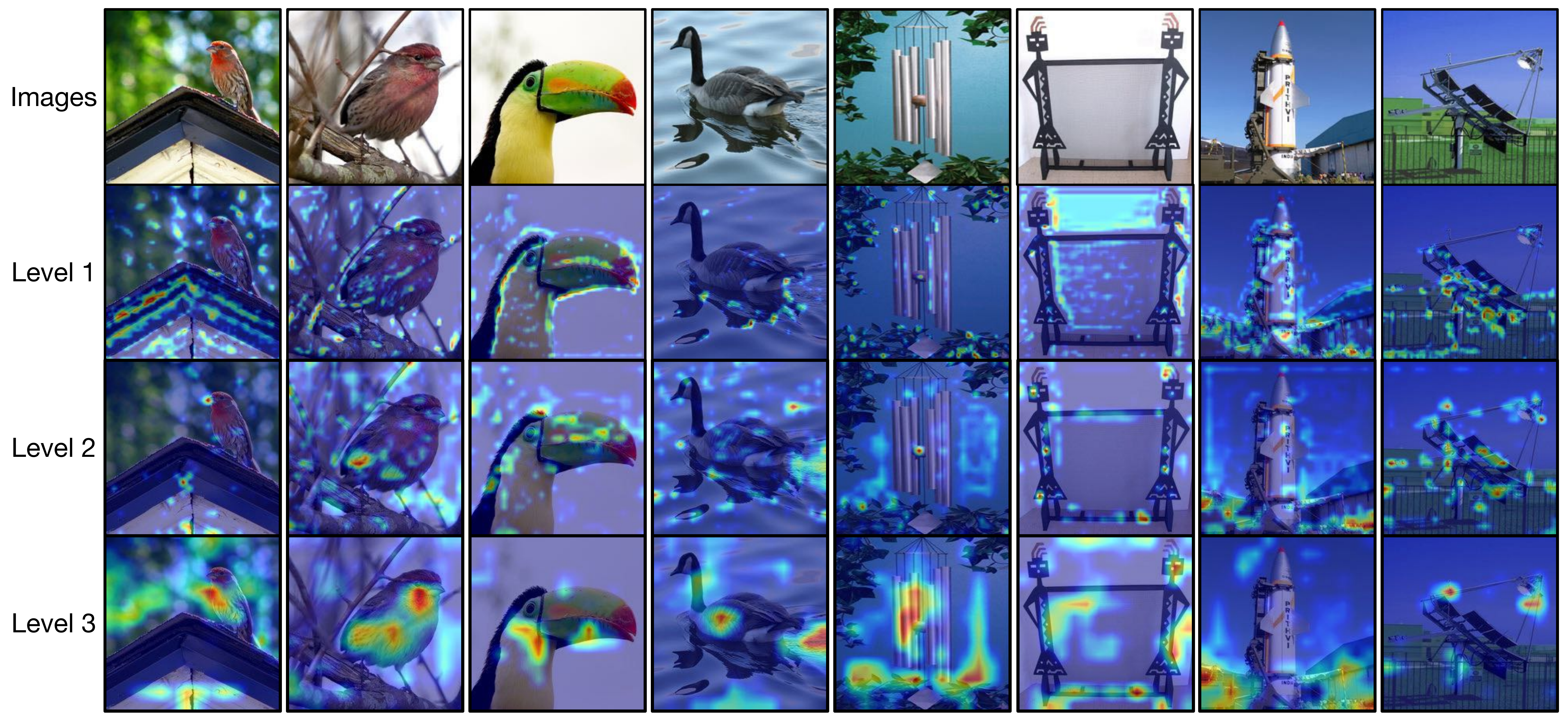}
	\caption{\small Visualization of features at several  levels of abstraction. We indicate the Gate Module scores (GM) in the right bottom corner per image. We use ResNet-12 FEN of MlSo to obtain features with the CAM \cite{cam} visualization model on images of \textit{mini}-ImageNet. The class-wise activation maps vary at different levels of abstraction. Given animals in rows 1--4, shallower features are localized at boundaries between the object and background. 
	More abstract features appear on salient fine-grained parts of animals. For images of aerospace airplanes and radars, shallower features come from objects, whereas more abstract features capture salient parts of objects or the context around them. Therefore, multiple levels of abstraction provide complementary descriptions of various visual concepts.} 
	\label{fig:cam}
\end{figure*}

\subsection{Relationship Descriptor \texorpdfstring{$\vartheta$}{V}}
Table \ref{tab_methods} offers some choices for the operator $\vartheta$ whose role is to capture and summarize the information stored in query-support embeddings of image or video pairs. Such a summary is then passed to SN for learning to compare query-support pairs. 
Operator ($\otimes$+F) concatenates the support and query feature vectors prior to the outer-product step. Operator ($\otimes$+R) performs the outer-product on the support and query feature vectors separately prior to concatenation step along the third mode. Operator ($\otimes$)  averages over embeddings of $Z$ support images from a given $\mathcal{W}_l$, followed by the outer-product step on the mean support and query vectors, and the concatenation step along the third mode. 

\subsection{Multi-level Relation Learning}
\label{sec:mlrl}

It is generally known that different layers of CNN are able to capture different kinds of visual abstraction \eg, early layers may respond to edges, blobs and texture patterns, whereas deeper layers may respond to parts of objects under some small deformations. Therefore, we propose the Multi-level Second-order Similarity Network whose goal is to learn the similarity between query-support image  or video pairs. To this end, we leverage the multiple levels of visual abstraction to learn object relations based on fine-to-coarse embeddings of images (or videos). 

The architecture of MlSo is shown in Figure \ref{fig:md_sosn}. In contrast to SoSN which uses the final output of FEN as embeddings, MlSo extracts feature vectors across fine-to-coarse levels that correspond to consecutive convolutional blocks, and forms autocorrelation matrices with PN and GM to produce the corresponding feature representations at multiple levels. Subsequently, an autocorrelation matrix per level is passed to SN with the goal of computing the relation scores which are fused to make the final prediction. 
Figure \ref{fig:cam}  visualizes  class activation maps at three different levels of FEN (ResNet-12 backbone) to demonstrate the complementary nature of activated regions. The figure shows that each level of visual abstraction in MlSo responds differently to the same visual stimulus.

Subsequently, we employ SN learners with parameters $\tR_d,\,d\in\idx{D}$, to learn the relations between query-support pairs for each level of abstraction. The $L$-way $Z$-shot learning step over $D$ levels of feature abstraction can be redefined as learning the similarity \wrt relation descriptors:
%
%
%
\setcounter{equation}{3}
\begin{align}
& \zeta^{(d)}_{lq}=r\left(\vartheta\!\left([\mPhi^{(d)}_n]_{n\in\mathcal{W}_l},\mPhi^{*(d)}_{q}\right),\tR_d\right), \\ 
& \quad \text{ where }\; \mPhi_n^{(d)}\!=\!f^{(d)}(\mX_n; \tF) \;\text{ and }\; \mPhi^{*(d)}_q\!=\!f^{(d)}(\mX^*_q\!; \tF), \nonumber
\end{align}
with the following MSE loss:
\begin{align}
\!\!\!&\argmin\limits_{\tF, \tR_1,\cdots,\tR_D} L_{R},\label{eq:sosn2}\\
\!\!\!&\;\text{ where } L_{R} =\!\!\!\!\!\expect\limits_{(\mathcal{W}_1,\cdots,\mathcal{W}_L,\,q)\sim\mathcal{E}}\sum\limits_{l\in\idx{L}}\sum\limits_{d\in\idx{D}} \left(\zeta^{(d)}_{lq}\! - \delta(c_l\!-\!c^*_q)\right)^2,\nonumber
\end{align}
and the label-related symbols are defined as for Eq. \eqref{eq:sosn1}. 
During the inference step, we average the individual classifier votes obtained from  multiple levels of feature abstraction. We determine the final class $c^*_q$  for a given  testing episode $(\mathcal{W}_1,\cdots,\mathcal{W}_L,\,q)\sim\mathcal{E}_{test}$ as follows:
\begin{align}
c^*_q\!=\!c_{l'} \;\text{ where \;} l'\!=\argmin\limits_{l\in\idx{L}} 
\sum\limits_{d\in\idx{D}} \left(S^{(d)}_{lq}-1\right)^2.
\label{eq:sosn3}
\end{align}

\vspace{0.1cm}
\noindent\textbf{Scale-wise Boosting (SB).}
{Below we demonstrate how the inputs at multiple spatial scales are utilized by our multi-level network. The use of scale-wise inputs is shown in Figure \ref{fig:md_sosn}. Given a pair of support and query images, $\mathbf{X}_n$ and $\mathbf{X}_q^*$, we first downsample them to obtain the multi-scale inputs, $\mathbf{X}_{n}^{(s)}$ 
and $\mathbf{X}_q^{*(s')}$. Then we feed them into the feature encoder to generate the convolutional features $\mPhi^{(s)}$ and $\mPhi^{*(s')}$, where $s$ and $s'$ are scale-wise indexes. Subsequently, we form the support-query relation descriptors at spatial scales $(s,s')$, pass them via function $r$, and obtain the scale-wise relation similarity scores $\zeta_{lq}^{(ss')}$:
\vspace{-0.5cm}
\begin{equation}
    \zeta_{lq}^{(ss')} = r(\vartheta(\mPhi^s_n,\mPhi^{*s'}_q), \tR),
\end{equation}
where $s,s'\in \idx{S}$ are the scale indexes ranging from $1$ to $3$, referring to inputs at 1, 1/2 and 1/4 of the original resolution. The number of scales we use is $S\!=\!3$.

\vspace{0.1cm}
As $\zeta_{lq}^{(ss')}$ is the similarity between $\mathbf{X}_n$ and $\mathbf{X}_q^*$ at  scales $s$ and $s'$,  the remaining part of the pipeline may deal with the scale-wise matching of objects in various ways. Note that the scale-wise matching deals with the scale variations of objects, whereas the feature abstraction relates to the level of semantic composition. Thus, both strategies may be combined to improve the few-shot learning step. 

Below is a simple extension of our pipeline to a variant with  multiple levels of abstraction and multiple spatial scales:
\begin{align}
& \zeta^{(d,ss')}_{lq}=r\left(\vartheta\!\left([\mPhi^{(d,s)}_n]_{n\in\mathcal{W}_l},\mPhi^{*(d,s')}_{q}\right),\tR_d\right),\label{eq:relrel}\\ 
& \quad \text{ where }\; \mPhi_n^{(d,s)}\!=\!f^{(d)}(\mX_n^s; \tF) \;\text{ and }\; \mPhi^{*(d,s')}_q\!=\!f^{(d)}(\mX^{*s'}_q\!; \tF), \nonumber
\end{align}
with the following MSE loss:
\begin{align}
&\argmin\limits_{\tF, \tR_1,\cdots,\tR_D} L_{R}, \quad\text{ where }\quad L_R =\label{eq:sosn22}\\
&\expect\limits_{(\mathcal{W}_1,\cdots,\mathcal{W}_L,q)\sim\mathcal{E}}\sum\limits_{l\in\idx{L}}\sum\limits_{d\in\idx{D}} \sum\limits_{s\in S}\sum\limits_{s'\in S}\frac{1}{ss'}\left(\zeta^{(d,ss')}_{lq}\! - \delta(c_l\!-\!c^*_q)\right)^2\!.\nonumber
\end{align}

Such a basic formulation  outperforms our earlier formulations. However, the relation network in Eq. \eqref{eq:relrel} and the loss in Eq. \eqref{eq:sosn22} can be modified to form a better strategy utilizing the abstraction level and scale matching steps, as detailed next.
}

\vspace{0.1cm}
\noindent\textbf{Feature Matching (FM).} Although MlSo can extract feature vectors at various spatial scales and levels of visual abstraction, not all scales and visual abstraction levels have the same importance. Figure \ref{fig:cam} shows that the features across different levels may contain foreground, context or even background information. In order to control the respective contributions of each scale and level of feature abstraction, we introduce a so-called Feature Matching. The FM step acts as an adaptive switch or weight for every branch of the relation learner. 

Moreover, FM  can be applied at the intra- or inter-level. Figure \ref{fig:matching} (a) shows that   FM can be applied within the single-level scale-wise features. Such a strategy is called intra-level matching as exclusive base learners are trained to make predictions at each respective level.  Alternatively, one can choose to apply  FM over multi-level representations, as shown in Figure \ref{fig:matching} (b). Such a strategy employs a shared base learner to make predictions over all available pairs.

For intra- and inter-matching, we investigate four types of matching: Cosine Matching (CM), Gate Module (GM), Optimal Transport (OT) and GRaph matching (GR). Let us  take the intra-level matching scheme as example. 

Let $\vpsi_n^{(d,s)}\!\!\equiv\!(\mPsi_n^{(d,s)})_{(:)}$ and $\vpsi_q^{*(d,s')}\!\!\equiv\!(\mPsi_q^{*(d,s')})_{(:)}$ be pooled (as in Eq. \eqref{eq:concat_best} but without the concatenation step) support and query descriptors, vectorized by the operator $(:)$.  

\vspace{0.1cm}
\noindent
\textbf{Cosine Matching (CM)} simply reweights scale-wise feature pairs via the cosine similarity  $\alpha_{nq}^{(d,ss')}\!=\!\langle\vpsi_n^{(d,s)},\vpsi_q^{*(d,s')}\rangle$. 

\vspace{0.1cm}
\noindent
\textbf{Gate Module (GM)} in Fig. \ref{fig:blocks} learns  to assign an attention score for each pair based on a convolutional network $g(\cdot)$ so that  $\alpha_{nq}^{(d,ss')}\!\!=\!g(\vpsi_n^{(d,s)})\!\cdot\!g(\vpsi_q^{*(d,s')})$.
For CM and GM, the updated relation score is defined as follows:
\setcounter{equation}{12}
\begin{align}
    \zeta'^{(d)}_{nq} = \alpha_{nq}^{(d,ss')}r(\vpsi_n^{(d,s)}, \vpsi_q^{*(d,s')}; \tR_d).
\end{align}
Subsequently, we form the following MSE loss:
\begin{align}
&\argmin\limits_{\tF, \tR_1,\cdots,\tR_D} L_{R} =\!\!\!\!\!\!\\ &\text{ where }L_{R}\label{eq:sosn_new1} =\!\!\!\expect\limits_{(\mathcal{W}_1,\cdots,\mathcal{W}_L,\,q)\sim\mathcal{E}}\sum\limits_{l\in\idx{L}}\sum\limits_{n\in\mathcal{W}_l}\sum\limits_{d\in\idx{D}}\! \left(\zeta^{'(d)}_{nq}\!-\!\delta(c_l\!-\!c^*_q)\right)^2,
\nonumber
\end{align}
where the label-related symbols are defined as for Eq. \eqref{eq:sosn1}.

\vspace{0.1cm}
\noindent
\textbf{Optimal Transport (OT)}. DeepEMD \cite{deepemd}  transports the support location-wise representations to match the query location-wise representations (location-wise matching). In contrast, we propose a strategy complementary to design of DeepEMD. We solve a linear program to transport the support intra-level scale-wise representations into the query intra-level scale-wise representations. We obtain the dominant matching pattern of spatial scales at the given  abstraction levels by:
\begin{align}
\vspace{-0.5cm}
    &\zeta^{'(d)}_{nq}\!=\!1-\min\limits_{\boldsymbol{\omega}^{(d)}_{nq}\geq 0}\;\;\sum\limits_{s,s'}\omega_{nq}^{(d,ss')}\alpha_{nq}^{(d,ss')}\!\!,\\
   &\qquad\qquad\;\;{\textstyle\text{\footnotesize s.t. }\scriptstyle\sum\limits_{s'}\omega_{nq}^{(d,ss')}=\Omega_n^{(d,s)},\;\; \sum\limits_s \omega_{nq}^{(d,ss')}=\Omega_q^{*(d,s')}\!\!},\nonumber\\
    &\text{where }\alpha_{nq}^{(d,ss')}\!\!=\!\langle\vpsi_n^{(d,s)},\vpsi_q^{*(d,s')}\rangle,\nonumber
\end{align}

\noindent
and $\Omega_n^{(d,s)}\!\!=\!\max(0,\langle\vpsi_n^{(d,s)}\!,\frac{1}{N_s}\sum_{s'=1}^{S} \vpsi_q^{*(d,s')}\rangle)$ and $\Omega_q^{*(d,s')}\!\!=\!\max(0,\langle\vpsi_q^{*(d,s')}, \frac{1}{N_s}\sum_{s=1}^{S} \vpsi_n^{(d,s)})$ constraint weights
$\omega_{nq}^{(d,ss')}$. 
Subsequently, we apply the loss in Eq. \eqref{eq:sosn_new1}.


\vspace{0.1cm}
\noindent
\textbf{GRaph matching (GR)} strategy is based on a Graph Neural Network (GNN), which learns correlations  between support and query samples at different levels of abstraction and scales. To implement GR, we form an adjacency matrix $\boldsymbol{\alpha}^{(d)}_{nq}\!\in\!\mbr{2S\times 2S}$ representing a weighted undirected graph capturing scales of the support-query pair for the abstraction level $d$.  
Following the standard GNN notation, we have  $\mathbf{H}^{(d,l'+1)}_{nq}\!\!=\!\sigma(\boldsymbol{\alpha}^{(d)}_{nq}\mathbf{H}^{(d,l')}_{nq}\mathbf{W}^{(l')})$, where $\sigma(\cdot)$ is a non-linearity and $l'$ refers to the GNN layer index. The node information matrix is given by $\mathbf{H}^{(d,1)}_{nq}\!\equiv\!\left[
[\vpsi^{(d,s)}_{n}]^T_{s\!\in\idx{S}};_1
[\vpsi^{*(d,s')}_{q}]^T_{s'\!\in\idx{S}}
\right]$, and $\mathbf{W}^{(l')}$ are filters of GNN. The adjacency matrix  $\boldsymbol{\alpha}$  is defined as: 

\comment
{\small
\begin{equation}
    \!\!\!\!\alpha_{ij} = \begin{cases}
    %
    \langle\mPsi_n^{(d,i\text{mod}\,S)},\mPsi_q^{(d,j\text{mod}\,S)}\rangle & \text{if}\; \text{\fontsize{7}{8}\selectfont $(1\leq\!s\!\leq 3\wedge 4\leq\!s'\!\leq 6) \vee (4\leq\!s\!\leq 6\wedge 1\leq\!s'\!\leq 3)$},\\
    0 & \text{otherwise}
    \end{cases}\!\!\!\!
\end{equation}
}

{
\vspace{-0.3cm}
\fontsize{7}{8}\selectfont
\begin{equation}
\!\!\!\boldsymbol{\alpha}^{(d)}_{nq}\!=\!\!\left[
\begin{array}{c|c}
\!\!\!\!\Big[\!\langle\mPsi_n^{(d,s)}\!\!,\mPsi_n^{(d,s')}\rangle\!\cdot\!e^{\frac{-(s\!-\!s')^2}{2\sigma'^2}}\Big]_{\substack{s\in\idx{S}\\s'\in\idx{S}}}\!\!\!\! & 
\!\!\!\!\Big[\langle\mPsi_n^{(d,s)}\!\!,\mPsi_q^{*(d,s')}\rangle\!\cdot\!e^{\frac{-(s\!-\!s')^2}{2\sigma'^2}}\Big]_{\substack{s\in\idx{S}\\s'\in\idx{S}}}\!\!\!\! \\
\hline
\!\!\!\!\Big[\!\langle\mPsi_q^{*(d,s)}\!\!,\mPsi_n^{(d,s')}\rangle\!\cdot\!e^{\frac{-(s\!-\!s')^2}{2\sigma'^2}}\Big]_{\substack{s\in\idx{S}\\s'\in\idx{S}}}\!\!\!\! &
\!\!\!\!\Big[\!\langle\mPsi_q^{*(d,s)}\!\!,\mPsi_q^{*(d,s')}\rangle\!\cdot\!e^{\frac{-(s\!-\!s')^2}{2\sigma'^2}}\Big]_{\substack{s\in\idx{S}\\s'\in\idx{S}}}\!\!\!\!
\end{array}
\right],
\end{equation}
\vspace{-0.2cm}
}

\noindent
where  the dot product $\langle\cdot,\cdot\rangle$ captures the visual similarity between representations at scales $s,s'\!\in\!\idx{S}$, and $e^{\frac{-(s\!-\!s')^2}{2\sigma'^2}}$ is the Radial Basis Function (RBF) similarity prior \wrt scales $s$ and $s'$. Moreover, $\sigma'\!=\!\frac{1}{3}(S\!-\!1)$ ensures that for the maximum difference of scales equal $S\!-\!1$, the RBF kernel decays by 90\% of its maximum value of one.

Subsequently, we apply a fully-connected layer to $\mathbf{H}^{(d,L')}$, where the last layer $L'\!=\!3$:
{
\begin{align}
\zeta_{nq}^{'(d)}\!=\!\text{MLP}\left(\mathbf{H}^{(d,L')}_{nq}\right), 
\end{align}
\vspace{-0.4cm}
}

\noindent
and then $\zeta_{nq}^{'(d)}$ can be substituted into Eq. \eqref{eq:sosn_new1}.

\vspace{0.1cm}
\noindent\textbf{Visual Abstraction Level and Scale Discriminator.} 
Self-supervised learning has become a mainstream tool in devising robust representations. 
Below we propose to employ a self-supervised Visual Abstraction Level and Scale Discriminator, called VALSD, whose role is to learn to predict the level of abstraction and the spatial scale indexes by observing the autocorrelation matrices of a given branch. Such a design is consistent with the idea of employing a simple pretext task. Self-supervised learning encourages the network to preserve additional information about images, leading to highly discriminative representations.

The VALSD unit is indicated in Figure \ref{fig:md_sosn}. We proceed by simply taking second-order representations $\vpsi^{(d,s)}_n$ and $\vpsi^{*(d,s)}_q$ for each level of abstraction $d\in\idx{D}$ and a spatial scale $s\!\in\!\idx{S}$, where $(d,s)$ indexes serve as labels for the pretext task. We employ the softmax cross-entropy classifier on $\vpsi^{(d,s)}_n$ and $\vpsi^{*(d,s)}_q$ with the goal of predicting the joint abstraction-scale label $d\!\cdot\!S\!+\!s$, which ranges from 1 to 9 for $S\!=\!3$ and three spatial scales.


%

\subsection{Unsupervised Pipeline}
\label{sec:un}
For completeness, below we introduce an unsupervised FSL pipeline which does not rely on any class-wise training  annotations. Inspired by  self-supervised learning, 
%
we apply extensive augmentations to available training images (or videos) to form the so-called positive set, and we label these instances  as similar.  The above step promotes FEN to learn invariance to augmentations in the embedding space. 

Specifically, given two  image inputs $\mX$ and $\mY$, we first apply random augmentations on these images \eg, rotation, flip, resized crop and the color jitter via operator $\text{Aug}(\cdot)$ which samples these transformations according to a uniform distribution. We obtain a set of $M$ augmented images:
\begin{align}
    \widehat{\mX}_n \sim \text{Aug}(\mX),\;\widehat{\mY}_q \sim \text{Aug}(\mY),\;n,q\in\idx{M}.   
\end{align}
Subsequently, we pass the augmented images to the feature encoder $f$ and obtain their embeddings. Equations below  assume $D$ levels of feature abstraction but using a single level may be achieved by dropping the superscript $(d)$ in what follows. For the augmented samples of $\mX$, we form  relation descriptors which represent positive pairs (similar instances). For  the augmented samples of $\mY$, we repeat the above step to also form positive pairs. We then form exhaustively the relation descriptors between the augmented samples of $\mX$ and $\mY$ which represent negative pairs (dissimilar instances). Finally, we obtain relation predictions $\boldsymbol{\zeta}^{(d)},\boldsymbol{\zeta}^{*(d)}\!\in\mbr{M\!\times\!M}$ from the relation network $r(\cdot)$ for the augmented samples of $\mX$ and $\mY$, respectively. We also obtain the relation predictions $\boldsymbol{\zeta}'^{(d)}\!\in\mbr{M\!\times\!M}$ evaluated between the augmented samples of $\mX$ and $\mY$. The above steps are realized by:
\begin{align}
\label{eq:enc}
&\mPhi^{(d)}_n \!=\! f^{(d)}(\widehat{\mX}_n; \tF),\; \mPhi^{*(d)}_q \!=\! f^{(d)}(\widehat{\mY}_q; \tF),\; n,q,n'\!,q'\!\!\in\!\idx{M},\;d\!\in\!\idx{D},\nonumber \\
&\zeta^{(d)}_{nn'}\;=\!r\left(\vartheta\!\left(\mPhi^{(d)}_n,\mPhi^{(d)}_{n'}\right); \tR_d\right),\;\nonumber\\
&\, \zeta'^{(d)}_{nq}\!=\! r\left(\vartheta\!\left(\mPhi^{(d)}_{n},\mPhi^{*(d)}_{q}\right); \tR_d\right),\;\nonumber\\
&\zeta^{*(d)}_{qq'}\!=\! r\left(\vartheta\!\left(\mPhi^{*(d)}_q,\mPhi^{*(d)}_{q'}\right); \tR_d\right).
\end{align}
Finally, we minimize the contrastive loss $L_{uns}$ \wrt $\tF$ and $\tR$ in order to push closer the positive embedding pairs of the augmented samples generated from the same image ($\mX$ followed by $\mY$)  and push apart the negative pairs of embeddings  of the augmented samples generated from $\mX$ and $\mY$ (two different images):
\comment{
\begin{align}
     \argmin\limits_{\tF, \tR_1,\cdots,\tR_D}\;\;
     \expect\limits_{(\mX,\mY)\sim\mathcal{E}_{uns}}\;\;
     \expect\limits_{\left\{(\widehat{\mX}_n,\widehat{\mY}_q)\sim \text{Aug}(\mX)\times \text{Aug}(\mY)\right\}_{n,q\in\idx{M}} }
     \;\sum\limits_{d\in\idx{D}}\;\parallel\!\mS^{(d)}\!-\!1\!\parallel^2_F + \parallel\!\mS^{*(d)}\!-1\!\parallel^2_F
     +\parallel\!\mS'^{(d)}\!\parallel^2_F,\label{eq:urn}
\end{align}}
{
\begin{align}
\!\!\!\!\!\!&\argmin\limits_{\tF, \tR_1,\cdots,\tR_D}
     \expect\limits_{(\mX,\mY)}
     \expect\limits_{(\widehat{\mX}_n,\widehat{\mY}_q)_{n,q\in\idx{M}} }
     \sum\limits_{d\in\idx{D}}\!\!\parallel\!\!\boldsymbol{\zeta}^{(d)}\!\!-\!1\!\!\parallel^2_F
     \!+\! \parallel\!\!\boldsymbol{\zeta}^{*(d)}\!\!-\!1\!\!\parallel^2_F
     \!+\!\parallel\!\!\boldsymbol{\zeta}'^{(d)}\!\!\parallel^2_F,\nonumber\\
\!\!\!\!\!\!&\;\;\text{ where } (\mX,\mY)\!\sim\!\mathcal{E}_{uns} \text{ and }(\widehat{\mX}_n, \widehat{\mY}_q)\!\sim\! \text{Aug}(\mX)\!\times\! \text{Aug}(\mY).\label{eq:urn} 
\end{align}}
During the inference step, we simply apply Eq. \eqref{eq:sosn3}.



\section{Experiments}

\begin{table}[t]
\centering
\caption{Evaluations on the Omniglot dataset. 
}
\label{table2}
\makebox[\linewidth]{
\setlength{\tabcolsep}{0.5em}
\renewcommand{\arraystretch}{0.60}
\begin{tabular}{lcccccc}
\toprule
Model & & Fine & \multicolumn{2}{c}{5-way accuracy} & \multicolumn{2}{c}{20-way accuracy} \\ 
&& Tune & 1-shot & 5-shot & 1-shot & 5-shot  \\ \hline
\textit{Conv. Siamese Nets} &\cite{koch2015siamese} & N & $96.7$ & $98.4$ & $88.0$ & $96.5$ \\
\textit{Conv. Siamese Nets} &\cite{koch2015siamese} & Y & $97.3$ & $98.4$ & $88.1$ & $97.0$ \\
\textit{Matching Net} &\cite{vinyals2016matching} & N & $98.1$ & $98.9$ & $93.8$ & $98.5$ \\
\textit{Prototypical Net} &\cite{snell2017prototypical} & N & $99.8$ & $99.7$ & $96.0$ & $98.9$ \\ 
\textit{MAML} & \cite{finn2017model} & Y & $98.7$ & $99.9$ & $95.8$ & $98.9$ \\ 
\textit{Relation Net} & \cite{sung2017learning} & N & $99.6$ & $99.8$ & $97.6$ & $99.1$ \\ 
\cdashline{1-2}
\textit{SoSN+PN} & & N & ${99.8}$ & ${99.9}$ &${98.3}$ & ${99.4}$ \\ 
\textit{MlSo+PN} & & N & $\mathbf{99.9}$ & $\mathbf{99.9}$ & $\mathbf{98.7}$ & $\mathbf{99.7}$ \\ 
\bottomrule
\end{tabular}}
\end{table}

\label{sec:exp}
Below we demonstrate the usefulness of our approach by evaluating it on the Omniglot \cite{lake_oneshot}, \textit{mini}-ImageNet \cite{vinyals2016matching} and \textit{tiered}-ImageNet \cite{ren18fewshotssl}  datasets, a recently proposed Open MIC dataset \cite{me_museum}, fine-grained  Flower102 \cite{Nilsback08}, CUB-200  \cite{WahCUB_200_2011} and Food-101 \cite{bossard14} datasets, and action recognition datasets such as HMDB51 \cite{Kuehne11}, UCF101 \cite{soomro2012ucf101} and \textit{mini}-MIT \cite{soomro2012ucf101}. 
We evaluate the performance of SoSN and MlSo in both supervised and unsupervised settings to demonstrate their superior performance compared to first-order representations in few-shot learning. 
%
We train our network with the Adam solver. The layer configurations of our SoSN model are shown in Figure \ref{fig:blocks}.  
The results are compared against several state-of-the-art methods in one- and few-shot learning.

\begin{table*}[t]
\centering
\caption{Evaluations on the \textit{mini}-ImageNet dataset (5-way accuracy). {Asterisk `*' highlights that we converted the supervised Prototypical Net and the supervised Relation Net into the unsupervised contrastive pipelines, U-Prototypical Net and U-Relation Net, using the same mechanism (described in Section \ref{sec:un}) as the one applied for U-SoSN and U-MlSo.} }
\label{table3}
\makebox[\linewidth]{
\setlength{\tabcolsep}{2em}
\begin{tabular}{lcccc}
\toprule
\multicolumn{2}{c}{Model} & Backbone & 1-shot & 5-shot \\ \hline
& & & \multicolumn{2}{c}{Supervised Pipelines} \\
\midrule
\textit{Prototypical Net} & \cite{snell2017prototypical} & Conv-4 & $49.42 \pm 0.78 $ & $68.20 \pm 0.66 $ \\ 
\textit{MAML} & \cite{finn2017model} & Conv-4 & $48.70 \pm 1.84 $ & $63.11 \pm 0.92 $ \\ 
\textit{Relation Net} & \cite{sung2017learning} & Conv-4 & $50.44 \pm 0.82 $ & $65.32 \pm 0.70 $  \\ 
\textit{LwoF} & \cite{lowf} & WRN &  $56.20 \pm 0.86$ & $72.81 \pm 0.62$ \\
\textit{GNN} & \cite{gnn} & Conv-4 & $50.30 $ & $66.40 $  \\ 
\textit{Reptile} & \cite{Nichol2018Reptile} & Conv-4 &  $49.97 \pm 0.32 $ & $65.99 \pm 0.58 $  \\
\textit{Saliency Net} & \cite{zhang2019few} & Conv-4 & $57.45 \pm 0.86 $ & $72.01 \pm 0.75 $  \\
\textit{MAML++} & \cite{Antoniou2019Htmaml} & Conv-4 & $52.15 \pm 0.26 $ & $68.32 \pm 0.44 $  \\
\textit{TPN} & \cite{liu2018learning} & Conv-4 & 55.51 & 69.86 \\
\textit{Boosting} & \cite{gidaris2019boosting} & Conv-4 & $53.63 \pm 0.43$ & $71.70 \pm 0.36$ \\
\textit{TAML} & \cite{TAML} & Conv-4 & $51.77 \pm 1.86$ & $65.60 \pm 0.93$\\
\textit{KTN} & \cite{KTN} & Conv-4 & $54.61 \pm 0.80$ & $71.21 \pm 0.66$\\
\textit{ArL} & \cite{arl} & Conv-4 & $57.48 \pm 0.65$ & $72.64 \pm 0.45$ \\
\textit{TADAM} & \cite{tadam} & ResNet-12 & $58.50 \pm 0.30$ & $76.70 \pm 0.30$ \\
\textit{Variational FSL} & \cite{Zhang2019VariationalFL} & ResNet-12 & $61.23 \pm 0.26$ & $77.69 \pm 0.17$ \\
\textit{DeepEMD} & \cite{deepemd} & ResNet-12 & $65.91 \pm 0.82$ & $82.41 \pm 0.56$ \\
\textit{MetaOptNet} & \cite{metaopt}             & ResNet-18 & $62.64 \pm 0.61$ & $78.63 \pm 0.46$ \\
\textit{PN + SSL} & \cite{su2020does} & ResNet-18 & 58.40 & 76.60 \\
\textit{SimpleShot} & \cite{wang2019simpleshot} & ResNet-18 & $63.10 \pm 0.20$ & $79.92 \pm 0.14$  \\
\textit{FLAT} & \cite{FLAT} & WRN-28-10 & $59.88 \pm 0.83$ & $77.14 \pm 0.59$\\
\cdashline{1-2}
\textit{SoSN($\otimes$+F)+PN} & Eq. \eqref{eq:concat_f} & Conv-4 & $50.57 \pm 0.84 $ & $65.91 \pm 0.71 $  \\
\textit{SoSN($\otimes$) (no PN)} &  Eq. \eqref{eq:concat_best} & Conv-4 & $50.88 \pm 0.85 $ & $66.71 \pm 0.67 $  \\
\textit{SoSN($\otimes$+R)+PN} &  Eq. \eqref{eq:concat_r} & Conv-4 & $52.96\pm 0.83 $ & $68.58 \pm 0.70 $  \\
\textit{SoSN($\otimes$)+PN} & Eq. \eqref{eq:concat_best} & Conv-4 & ${52.96\pm 0.83 }$ & ${68.63 \pm 0.68}$  \\
\textit{MlSo($\otimes$)+PN} & Eq. \eqref{eq:concat_best} & Conv-4 & $\mathbf{58.03 \pm 0.80}$ & $\mathbf{73.06 \pm 0.72}$\\
\cdashline{1-2}
\textit{SoSN($\otimes$)+PN} & Eq. \eqref{eq:concat_best} & ResNet-12 & ${58.26\pm 0.87 }$ & ${73.20 \pm 0.68}$  \\
\textit{MlSo($\otimes$)+PN} & Eq. \eqref{eq:concat_best} & ResNet-12 & {$\mathbf{66.08\pm 0.85}$} & {$\mathbf{82.32 \pm 0.66}$}  \\
\textit{MlSo($\otimes$)+PN} & Eq. \eqref{eq:concat_best} & ResNet-18 & {$\mathbf{66.79 \pm 0.86}$} & {$\mathbf{83.41 \pm 0.67}$}  \\
\midrule
& & & \multicolumn{2}{c}{Unsupervised Pipelines} \\ \midrule
\multicolumn{2}{l}{\textit{Pixel (Cosine)}} & - & $23.00$ & $26.60$  \\
{\textit{BiGAN ($k_{nn}$)}} &  \cite{bigan} & - & $25.56$ & $31.10$  \\
\textit{BiGAN (cluster match)} & \cite{bigan} & - & $24.63$ & $29.49$  \\
\textit{DeepCluster ($k_{nn}$)}& \cite{deepcluster} & - & $28.90$ & $42.25$  \\
{\textit{DeepCluster (cluster match)}} & \cite{deepcluster} & - & $22.20$ & $23.50$  \\
{\textit{U-Prototypical Net}}\textsuperscript{*} & \cite{snell2017prototypical} & Conv-4 &  $35.85$ & $48.01$\\
{\textit{U-Relation Net}}\textsuperscript{*} & \cite{sung2017learning} & Conv-4 &  $35.14$ & $44.10$\\
\textit{UMTRA}& \cite{umtra} & Conv-4 & $39.91$ & $ 50.70$  \\
\textit{CACTUs} &\cite{cactu} & Conv-4 & $39.94$ & $ 54.01$  \\
\arrayrulecolor{black}
\cdashline{1-2}
\multicolumn{2}{l}{\textit{U-SoSN($\otimes$)+PN}} & Conv-4 & $37.94$ & $50.95$ \\
\multicolumn{2}{l}{\textit{U-MlSo+PN}} & Conv-4 & {$\mathbf{41.09}$} & {$\mathbf{55.38}$} \\
\bottomrule
\end{tabular}}
\end{table*}

\begin{table}[t]
\centering
\caption{Evaluations on the \textit{tiered}-ImageNet dataset (5-way accuracy) with the Conv-4 backbone.}
\label{table_tiered}
\makebox[\linewidth]{
\fontsize{8.5}{9}\selectfont
\begin{tabular}{lccc}
\toprule
Model && 1-shot & 5-shot \\ \hline
\textit{Incremental}&\cite{incremental}& $51.12 \pm 0.45$ & $66.40 \pm 0.36 $ \\
\textit{Soft k-means}&\cite{ren18fewshotssl}& $52.39 \pm 0.44 $ & $69.88 \pm 0.20 $ \\
\textit{MAML}&\cite{finn2017model}& $51.67 \pm 1.81 $ & $70.30 \pm 0.08 $ \\
\textit{Reptile}&\cite{Nichol2018Reptile}& $48.97 \pm 0.21 $ & $66.47 \pm 0.21 $ \\
\textit{Prototypical Net} & \cite{snell2017prototypical} & $53.31 \pm 0.89 $ & $72.69 \pm 0.74 $ \\ 
\textit{Relation Net} & \cite{sung2017learning} & $54.48 \pm 0.93$ & $71.32 \pm 0.78$ \\
\textit{TPN} &\cite{liu2018learning} & $57.41 \pm 0.94$ & $71.55 \pm 0.74$ \\
\cdashline{1-2}
\textit{SoSN} & & $58.62 \pm 0.92$ & $75.19 \pm 0.79$ \\
\textit{MlSo} && $\mathbf{61.97 \pm 0.91}$ & $\mathbf{78.83 \pm 0.77}$ \\
\midrule
& \multicolumn{2}{c}{Unsupervised Pipelines} \\ \midrule
\textit{U-Prototypical Net} & \cite{snell2017prototypical} & $37.52 \pm 0.93$ & $51.03 \pm 0.84$ \\
\textit{U-Relation Net} & \cite{sung2017learning} & $37.23 \pm 0.94$ & $49.54 \pm 0.83$ \\

\cdashline{1-2}
\textit{U-SoSN} & & $41.59 \pm 0.92$ & $55.81 \pm 0.76$ \\
\textit{U-MlSo} & & $\mathbf{43.01 \pm 0.91}$ & $\mathbf{57.53 \pm 0.74}$ \\ 
\bottomrule
\end{tabular}}
\end{table}

\comment{
\begin{table*}[t]
\centering
\caption{Ablations \wrt the number of abstraction levels $L$ on the \textit{mini}-ImageNet dataset (5-way accuracy) with the Conv-4 backbone. 
}
\label{table_depth}
\makebox[\textwidth]{
\fontsize{8.5}{9}\selectfont
\begin{tabular}{l|ccc|ccc}
\toprule
Model & \multicolumn{3}{c|}{1-shot} & \multicolumn{3}{c}{5-shot} \\ \hline
\# of levels $L$ & 2 & 3 & 4 & 2 & 3 & 4 \\ \hline
\textit{MlSo} & $55.24 \pm 0.81$ & $55.66\pm 0.83$ & $\mathbf{55.93 \pm 0.85}$ & $70.85 \pm 0.72$ & $70.88 \pm 0.70$ & $\mathbf{71.05 \pm 0.71}$  \\
\textit{MlSo+VALD} & $56.07\pm 0.82$ & $56.10 \pm 0.81$ & $\mathbf{56.31 \pm 0.83}$ & $\mathbf{72.07 \pm 0.69}$ & $71.65 \pm 0.70$ & $71.52 \pm 0.71$ \\
\textit{MlSo+GM} & $\mathbf{56.48\pm 0.83}$ & $ 56.31\pm 0.81 $ & $56.11 \pm 0.84$ & $\mathbf{71.67 \pm 0.70}$ & $71.31 \pm 0.71$ & $71.23 \pm 0.73$ \\
\textit{MlSo+VALD+GM} & $56.72\pm 0.82$ & $\mathbf{56.89 \pm 0.81}$ & $56.83 \pm 0.82$ & $72.31 \pm 0.72$ & $\mathbf{72.56 \pm 0.71}$ & $72.31 \pm 0.70$\\
\textit{MlSo+VALD+GM} & $56.72\pm 0.82$ & $\mathbf{56.89 \pm 0.81}$ & $56.83 \pm 0.82$ & $72.31 \pm 0.72$ & $\mathbf{72.56 \pm 0.71}$ & $72.31 \pm 0.70$\\
\bottomrule
\end{tabular}}
\end{table*}}

\begin{table*}[t]
\centering
\caption{\hg Evaluations on the fine-grained classification datasets (5-way accuracy) with the Conv-4 backbone.
}
\label{table5}
\makebox[\textwidth]{
\begin{tabular}{lccccccc}
\toprule
&& \multicolumn{2}{c}{CUB Birds} & \multicolumn{2}{c}{Stanford Dogs} & \multicolumn{2}{c}{Stanford Cars} \\
Model && 1-shot & 5-shot & 1-shot & 5-shot & 1-shot & 5-shot \\ \hline
\textit{\hg Matching Net} & \cite{vinyals2016matching} & \hg $45.30 \pm 1.03$ & \hg $59.50 \pm 1.01$ & \hg $35.80 \pm 0.99$ & \hg $47.50 \pm 1.03$ & \hg $34.80 \pm 0.98$ & \hg $44.70 \pm 1.03$ \\ 
\textit{\hg MAML} & \cite{finn2017model} & \hg $58.13 \pm 0.36$ & \hg $71.51 \pm 0.30$ & \hg $44.84 \pm 0.31$ & \hg $58.61 \pm 0.30$ & \hg $47.25 \pm 0.30$ & \hg $61.11 \pm 0.29$ \\ 
\textit{\hg Proto. Net} & \cite{snell2017prototypical} & \hg $37.36 \pm 1.00$ & \hg $45.28 \pm 1.03$ & \hg $37.59 \pm 1.00$ & \hg $48.19 \pm 1.03$ & \hg $40.90 \pm 1.01$ & \hg $52.93 \pm 1.03$  \\ 
\textit{\hg Relation Net} & \cite{sung2017learning} & \hg $58.99 \pm 0.52$ & \hg $71.20 \pm 0.40$ & \hg $43.29 \pm 0.46$ & \hg $55.15 \pm 0.39$ & \hg $47.79 \pm 0.49$ & \hg $60.60 \pm 0.41$  \\ 
\textit{\hg LRPABN} & \cite{huang2020low} & \hg $63.63 \pm 0.77$ & \hg $76.06 \pm 0.58$ & \hg $45.72 \pm 0.75$ & \hg $60.94 \pm 0.66$ & \hg $60.28 \pm 0.76$ & \hg $73.29 \pm 0.63$ \\
\textit{\hg MattML} & \cite{zhu2020multi} & \hg $66.29 \pm 0.56$ & \hg $80.34 \pm 0.30$ & \hg $54.84 \pm 0.53$ & \hg $71.34 \pm 0.38$ & \hg $66.11 \pm 0.54$ & \hg $82.80 \pm 0.28$ \\
{\textit{\hg SoSN}} & \cite{sosn} & \hg ${64.56 \pm 0.91}$ & \hg ${77.82 \pm 0.57}$ & \hg ${48.21 \pm 0.72}$ & \hg ${63.15 \pm 0.67}$ & \hg $62.88 \pm 0.72$ & \hg ${76.10 \pm 0.58}$ \\ 
\cdashline{1-8}
\multicolumn{2}{l}{\textit{\hg MlSo}} & \hg {$\mathbf{68.21 \pm 0.78}$} & \hg {$\mathbf{82.18 \pm 0.47}$} & \hg {$\mathbf{55.62 \pm 0.58}$} & \hg {$\mathbf{71.98 \pm 0.71}$} & \hg {$\mathbf{67.83 \pm 0.63}$} & \hg {$\mathbf{84.98 \pm 0.48}$}  \\ 
\midrule
\textit{\hg U-Proto. Net} & \cite{sung2017learning} & \hg $34.51 \pm 0.53$ & \hg $56.32 \pm 0.41$ & \hg $32.05 \pm 0.49$ & \hg $43.96 \pm 0.48$ & \hg $33.87 \pm 0.57$ & \hg $48.20 \pm 0.46$  \\ 
\textit{\hg U-Relation Net} & \cite{sung2017learning} &  \hg $35.42 \pm 0.55$ & \hg $57.96 \pm 0.43$ & \hg $32.75 \pm 0.49$ & \hg $44.37 \pm 0.46$ & \hg $34.43 \pm 0.54$ & \hg $48.71 \pm 0.45$  \\ 
{\textit{\hg U-SoSN}} & \cite{sosn} & \hg ${43.14 \pm 0.51}$ & \hg ${65.02 \pm 0.43}$ & \hg ${41.56 \pm 0.49}$ & \hg ${53.62 \pm 0.47}$ & \hg ${40.31 \pm 0.55}$ & \hg ${57.98 \pm 0.43}$ \\ 
\cdashline{1-8}
\multicolumn{2}{l}{\textit{\hg U-MlSo}} & \hg {$\mathbf{46.31 \pm 0.53}$} & \hg {$\mathbf{68.67 \pm 0.46}$} & \hg {$\mathbf{45.02 \pm 0.48}$} & \hg {$\mathbf{56.89 \pm 0.46}$} & \hg {$\mathbf{43.81 \pm 0.56}$} & \hg {$\mathbf{61.13 \pm 0.41}$} \\
\bottomrule
\end{tabular}}
\end{table*}

\begin{table*}[t]
\centering
\caption{\hg Evaluations on the action recognition benchmarks (5-way accuracy). To fairly compare our method with the prior works, we follow the same evaluation splits with \cite{arn}. For our baselines, we re-implement the Prototypical Net and the Relation Net both with the use of 3D convolutions.}
\label{table6}
\makebox[\textwidth]{
\begin{tabular}{lccccccccc}
\toprule
&& \multicolumn{2}{c}{HMDB51} & \multicolumn{2}{c}{UCF101} & \multicolumn{2}{c}{{\em mini}-MIT} & \multicolumn{2}{c}{Kinetics} \\
Model && 1-shot & 5-shot & 1-shot & 5-shot & 1-shot & 5-shot & 1-shot & 5-shot \\ \hline
\textit{\hg C3D PN} & \cite{snell2017prototypical} & \hg ${38.05 \pm 0.97}$ & ${53.15 \pm 0.90}$ & \hg ${57.05 \pm 1.02}$ & \hg ${78.25 \pm 0.73}$ & \hg ${33.65 \pm 1.01}$ & \hg ${45.1 \pm 0.90}$ & \hg 57.11 & \hg 77.92 \\ 
\textit{\hg C3D RN} & \cite{sung2017learning} & \hg ${38.23 \pm 0.97}$ & \hg ${53.17 \pm 0.86}$ & \hg ${58.21 \pm 1.02}$ & \hg ${78.35 \pm 0.72}$ & \hg ${35.71 \pm 1.02}$ & \hg ${47.32 \pm 0.91}$ & \hg 56.98 & \hg 77.83\\
{\textit{\hg C3D SoSN}} & \cite{sosn} & \hg ${40.83 \pm 0.96}$ & \hg ${55.18 \pm 0.86}$ & \hg ${62.57 \pm 1.05}$ & \hg ${81.51 \pm 0.75}$ & \hg ${40.83 \pm 0.99}$ & \hg ${52.16 \pm 0.95}$ & \hg 58.77 & \hg 79.02 \\
{\textit{\hg CMN}} & \cite{cmn} & \hg -& \hg -& \hg -& \hg -& \hg -& \hg -& \hg 60.50 & \hg 78.90 \\
{\textit{\hg ARN}} & \cite{arn} & \hg ${45.52 \pm 0.96}$ & \hg ${58.96 \pm 0.87}$ & \hg ${66.32 \pm 0.99}$ & \hg ${83.12 \pm 0.70}$ & \hg ${43.05 \pm 0.97}$ & \hg ${56.71 \pm 0.87}$ & \hg 63.70 & \hg 82.40 \\ 
\cdashline{1-10}
\multicolumn{2}{l}{\textit{\hg C3D MlSo}} & \hg $\mathbf{46.69 \pm 0.93}$ & \hg $\mathbf{60.31 \pm 0.83}$ & \hg $\mathbf{68.19 \pm 0.95}$ & \hg $\mathbf{87.11 \pm 0.71}$ & \hg $\mathbf{44.67 \pm 0.95}$ & \hg $\mathbf{58.68 \pm 0.86}$ & \hg $\mathbf{66.32}$ & \hg $\mathbf{85.21}$  \\
\bottomrule
\end{tabular}}
\end{table*}

{
\begin{table*}[t]
\centering
\caption{Evaluations on the Open MIC dataset (Protocol I) (5-way 1-shot accuracy). 
}
\label{table_o1}
\makebox[\textwidth]{
\setlength{\tabcolsep}{0.4em}
\renewcommand{\arraystretch}{0.90}
\begin{tabular}{lccccccccccccc}
\toprule
Model & & $p1\!\!\rightarrow\!\!p2$ & $p1\!\!\rightarrow\!\!p3$& $p1\!\!\rightarrow\!\!p4$& $p2\!\!\rightarrow\!\!p1$& $p2\!\!\rightarrow\!\!p3$ &$p2\!\!\rightarrow\!\!p4$& $p3\!\!\rightarrow\!\!p1$& $p3\!\!\rightarrow\!\!p2$& $p3\!\!\rightarrow\!\!p4$& $p4\!\!\rightarrow\!\!p1$& $p4\!\!\rightarrow\!\!p2$& $p4\!\!\rightarrow\!\!p3$\\
\hline
\it Matching Net& \cite{vinyals2016matching} & $68.7$ & $52.3$ & $60.9  $ & $45.6  $ & $48.1  $ & $69.1 $ & $48.0 $ & $46.9 $ & $65.7  $ & $42.1  $ & $68.9  $ & $50.5  $\\
\it MAML & \cite{finn2017model} & $69.1  $ & $52.8  $ & $61.8  $ & $46.4  $ & $48.2  $ & $69.1 $ & $48.4 $ & $48.3 $ & $65.8 $ & $43.2  $ & $70.2  $ & $50.2  $\\
\it Reptile & \cite{Nichol2018Reptile} & $69.2  $ & $52.5  $ & $62.2  $ & $47.1  $ & $48.8  $ & $69.3 $ & $48.5 $ & $48.7 $ & $66.1 $ & $43.6  $ & $70.1  $ & $50.3  $\\
\it Proto. Net& \cite{snell2017prototypical} & $70.0  $ & $53.9  $ & $62.1  $ & $46.5  $ & $49.7  $ & $69.9 $ & $49.1 $ & $48.2 $ & $67.1  $ & $43.9  $ & $70.5  $ & $51.1  $\\
\it Relation Net& \cite{sung2017learning} & $71.1  $ & $53.6  $ & $63.5  $ & $47.2  $ & $50.6  $ & $68.5  $ & $48.5  $ & $49.7  $ & $68.4  $ & $45.5  $ & $70.3  $ & $50.8  $\\
\it SoSN & \cite{sosn} & ${81.4}   $ & ${65.2}   $ & ${75.1}   $ & ${60.3}   $ & ${62.1}   $ & ${77.7}   $ & ${61.5}   $ & ${82.0}   $ & ${78.0}   $ & ${59.0}   $ & ${80.8}   $ & ${62.5} $\\
\multicolumn{2}{l}{\it MlSo} & $\mathbf{81.6}   $ & $\mathbf{66.3} $ & $\mathbf{77.2} $ & $\mathbf{62.4} $ & $\mathbf{63.1} $ & $\mathbf{78.2} $ & $\mathbf{64.2} $ & $\mathbf{81.4} $ & $\mathbf{78.0} $ & $\mathbf{60.1} $ & $\mathbf{81.6} $ & $\mathbf{63.6} $\\ 
\midrule
\it U-Relation Net & \cite{sung2017learning} & ${67.1}   $ & ${48.1}   $ & ${62.5}   $ & ${41.2}   $ & ${45.3}   $ & ${58.0}   $ & ${50.1}   $ & ${57.5}   $ & ${53.8}   $ & ${46.9}   $ & ${66.1}   $ & ${43.3}   $\\
\it U-SoSN & \cite{sosn} & ${78.6}   $ & ${58.8}   $ & ${74.3}   $ & ${61.1}   $ & ${57.9}   $ & ${72.4}   $ & ${62.3}   $ & ${75.6}   $ & ${73.7}   $ & ${58.5}   $ & ${76.5}   $ & ${54.6}   $\\
\multicolumn{2}{l}{\it U-MlSo} & $\mathbf{80.9}   $ & $\mathbf{61.5}   $ & $\mathbf{76.3}$ & $\mathbf{62.0}$ & $\mathbf{60.3}$ & $\mathbf{75.1}$ & $\mathbf{64.1}$ & $\mathbf{77.9}$ & $\mathbf{76.2}   $ & $\mathbf{59.9}$ & ${79.1}$ & $\mathbf{58.2}$\\
\bottomrule
\end{tabular}}
p1: shn+hon+clv, p2: clk+gls+scl, p3: sci+nat, p4: shx+rlc. Notation {\em x$\!\rightarrow$y} means training on exhibition {\em x} and testing on {\em y}.
\end{table*}
}

\subsection{Datasets}
Below we describe our setup, as well as the category recognition, fine-grained and action recognition datasets.

\vspace{0.05cm}
\noindent\textbf{Omniglot } \cite{lake_oneshot} consists of 1623 characters  from 50 alphabets. Samples in each class are drawn by 20 different people. The dataset is split into 1200 classes for training and 423 classes for testing. All images are resized to $28\!\times\!28$ pixels.

\vspace{0.05cm}
\noindent\textbf{\textit{mini}-ImageNet} \cite{vinyals2016matching} consists of 60000 RGB images from 100 classes, each class containing 600 samples. 
We follow the standard protocol \cite{vinyals2016matching} and use 80 classes for training (16 classes selected for validation) and remaining 20 classes for testing. We  use images of $84\!\times\!84$ pixels.


\vspace{0.05cm}
\noindent\textbf{\textit{tiered}-ImageNet} \cite{ren18fewshotssl} consists of 608 classes from ImageNet. We follow the protocol with 351 base classes, 96 validation classes and 160 novel test classes.

\vspace{0.05cm}
\noindent\textbf{Open MIC} stands for the Open Museum Identification Challenge (Open MIC) \cite{me_museum}, a recent dataset with photos of various exhibits \eg, paintings, timepieces, sculptures, glassware, relics, science exhibits, natural history pieces, ceramics, pottery, tools and indigenous crafts, captured within 10 museum exhibition spaces according to which this dataset is divided into 10 sub-problems. In total, Open MIC has 866 diverse classes and 1--20 images per class. The within-class images undergo various geometric and photometric distortions as the data was captured with wearable cameras. This makes Open MIC a perfect candidate for testing one-shot learning algorithms. We combine ({\em shn+hon+clv}), ({\em clk+gls+scl}), ({\em sci+nat}) and ({\em shx+rlc}) into sub-problems {\em p1}, $\!\cdots$, {\em p4}. We form 12 possible pairs in which sub-problem $x$ is used for training and sub-problem $y$ is used for testing (x$\rightarrow$y).

\vspace{0.05cm}
\comment{
\noindent\textbf{Flower102} \cite{Nilsback08} is a fine-grained category recognition dataset that contains 102 classes of various flowers. Each class consists of 40-258 images. We randomly select 80 classes for training and 22 classes for testing.

\vspace{0.05cm}
\noindent\textbf{Food-101} \cite{bossard14} consists of 101000 images in total and 1000 images per category. We choose 80 classes for training and 21 classes for testing.

}
{\hg
\vspace{0.05cm}
\noindent\textbf{Caltech-UCSD-Birds 200-2011 (CUB Birds)} \cite{WahCUB_200_2011} has 11788 images of 200 bird species. We follow the splits from  \cite{huang2020low}, that is, 130 classes are selected for training, 20 classes for validation and the remaining 50 categories for testing.

\vspace{0.05cm}
\noindent\textbf{Stanford Dogs} \cite{cars} has 17150 instances of 120 dogs classes where 70 classes are used for training, 20 classes for validation and the remaining 30 classes for testing, which is consistent with the protocol in \cite{huang2020low}.

\vspace{0.05cm}
\noindent\textbf{Stanford Cars} \cite{dogs} dataset has 16190 samples of 196 car categories. Following the protocol in \cite{huang2020low}, we use the 130/17/49 class splits for training, validation and testing, respectively.

\vspace{0.05cm}
\noindent\textbf{HMDB51} \cite{Kuehne11}, an action recognition dataset, contains 6849 clips divided into 51 action categories, each with a minimum of 101 clips. Following the protocol in \cite{arn}, 31  classes are selected for training, 10  classes are used for validating and the remaining 10  classes are used for testing. 

\vspace{0.05cm}
\noindent\textbf{UCF101} \cite{soomro2012ucf101} contains  action videos collected from YouTube. It has 13320 video clips and 101 action classes. Following the protocol in \cite{arn}, we pick 70 classes for training, 10 classes for validation and the remaining 21 classes for testing.

\vspace{0.05cm}
\noindent\textbf{\textit{mini}-MIT} \cite{monfortmoments} is a subset of the newly proposed large-scale Moments in Time dataset. This mini version contains 200 classes with 550 videos per class. Following the protocol in \cite{arn}, we select 120 classes for training, 40 classes for validation and the rest 40 classes for testing.
}

\subsection{Experimental Setup}
For the Omniglot dataset, we follow the setup in \cite{sung2017learning}. 
For {\em mini}-ImageNet, we use the  5-way 1-shot and  5-way 5-shot protocols. For every training and testing episodes, we randomly select 5 and 3 query samples per class, respectively. We compute an average over 600 episodes to obtain results. 
We use the initial learning rate of $1e\!-\!3$ and train the model with $200000$ episodes. For \textit{tiered}-ImageNet, we follow the  settings used for the {\em mini}-ImageNet dataset. 
For Open MIC, we mean-center images per sub-problem. We use the initial learning rate of $1e\!-\!4$ and  train the network with $15000$ episodes. 
For the fine-grained classification datasets, we evaluate our models on the 5-way 1-shot and  5-way 5-shot protocols. The numbers of support and query samples in each episode are the same as in the {\em mini}-ImageNet setting. We use the initial learning rate of $1e\!-\!3$ and train the model with $200000$ episodes. 
For the action recognition datasets, we randomly sample 50 frames per video along the temporal mode. We resize video frames to $84\!\times\!84$ pixels, which results in a lightweight model. We evaluate our algorithm on the 5-way 1-shot and  5-way 5-shot protocols on the three datasets detailed earlier. We adopt the hyper-parameter configuration used on the {\em mini}-ImageNet dataset.

\subsection{Evaluation Results}
Below we evaluate SoSN and its multi-level extension MlSo, and we compare them with state-of-the-art methods on datasets introduced above. 

\vspace{0.05cm}
\noindent\textbf{Omniglot.} Table \ref{table2} shows rather saturated results. We consider experiments on the Omniglot dataset as a sanity check to validate the performance of SoSN in the best default setting, that is, using the relationship descriptor from Eq. \eqref{eq:concat_best} and Power Normalization from Proposition \ref{pr:axmin}. {We note that the basic multi-level
 multi-scale variant of MlSo based on Eq. \eqref{eq:sosn22} outperforms SoSN.}

\vspace{0.05cm}
\noindent\textbf{{\em mini}-ImageNet.} Table \ref{table3} demonstrates that our method outperforms other approaches on both 1- and 5-shot evaluation protocols.
Firstly, we note that comparisons between various relation descriptors with/without PN are included and discussed as ablation studies in Section \ref{sec:abl}. 

For the image size of $84\!\times\!84$ and the 5-way 1-shot experiment, our best singe-level SoSN model achieves $\sim\!2.5\%$ higher accuracy than Relation Net \cite{sung2017learning}. Our best singe-level SoSN also outperforms Prototypical Net by $\sim\!3.5\%$ accuracy on the 5-way 1-shot protocol. Not shown in the table are results for SoSN trained with images of $224\!\times\!224$ pixels, in which case the accuracy scores on both protocols increase by $5.45\%$ and $4.33\%$, respectively. Such an accuracy gain demonstrates that SoSN benefits from large image sizes as second-order matrices are of higher rank for higher image resolutions due to the higher spatial resolution of feature maps compared with the low-resolution counterparts. Our similarity learning network works with variable resolutions of input images because SoSN operates on matrices whose size depends only on the number of output channels of encoding network.  

{Furthermore, our variant of MlSo with the ResNet-12 backbone, based on the inter-level matching strategy of the multi-level multi-scale feature representations by the Gate Module, achieved a significant gain between 6\% and 9\% accuracy over SoSN.} 
For unsupervised FSL, Table \ref{table3} shows that our U-SoSN and U-MlSo significantly outperform other unsupervised baselines (`{\em U-}' stands for the unsupervised FSL setting) \eg, U-Relation Net and U-Prototypical Net. { Compared to the U-SoSN model, the U-MlSo model  improves the top-1 accuracy by another 2\% and 3\% on the 1- and 5-shot protocols, respectively. }

{We  discuss the ablations on MlSo in Section \ref{sec:abl}. Firstly, we  discuss the results of the single-level SoSN model on more datasets. In the following tables, we drop `{\em ($\otimes$)+PN}' from our notations but this particular relation descriptor with PN is the best performing variant, thus it is used across the remaining experiments.}

\vspace{0.05cm}
\noindent\textbf{{\em tiered-}ImageNet.} Table \ref{table_tiered} shows the performance of our proposed methods on {\em tiered-}ImageNet. Our SoSN achieves $\sim\!4.1\%$ and $3.9\%$ improvement in accuracy, compared with Relation Net for 1- and 5-shot protocols, respectively. { Moreover, MlSo with the Gate Module achieves $\sim\!7.0\%$ and $3.9\%$ improvement in accuracy, compared with Relation Net for 1- and 5-shot protocols, respectively. 
The unsupervised variant, U-MlSo, yields $\sim\!6\%$ and $8\%$ gain in accuracy over U-Relation Net for  1- and 5-shot protocols, respectively.
Our supervised and unsupervised FSL models outperform all other  FSL approaches based on the Conv-4 backbone.}

\vspace{0.05cm}
\noindent\textbf{Open MIC.} Table \ref{table_o1} demonstrates that our single-level SoSN model outperforms the Relation Net \cite{sung2017learning} for all train/test sub-problems of Protocol I. For the 5- and 20-way protocols, Relation Net scores $55.45\%$ and $31.58\%$ accuracy, respectively. In contrast, our single-level SoSN scores $70.46\%$ and $49.05\%$ accuracy, respectively. 

\begin{table}[t]
\centering
\caption{\small Ablation study on {\em mini}-ImageNet \wrt different modules of our pipeline  (5-way accuracy, the `Conv-4-64' backbone, 4 stages and 3 scales, and OT matching were used.)}
\label{table_ablation_3}
\makebox[\linewidth]{
\begin{tabular}{cccc|cc}
\toprule
 Baseline &  SB &  FM &  VALSD &  1-shot &  5-shot \\  \hline
 \checkmark & & & &  55.93 &  71.05 \\
 \checkmark &  \checkmark & & &  57.28 &  72.01 \\
 \checkmark &  \checkmark &  \checkmark & &  57.79 &  72.65 \\
 \checkmark &  \checkmark &  \checkmark &  \checkmark &  \bf 58.03 &  \bf 73.06 \\
\bottomrule
\end{tabular}}
\end{table}

\begin{table}[t]
\centering
\caption{\small  Ablation study on {\em mini}-ImageNet \wrt the number of abstraction levels and spatial scales (1-shot/5-shot accuracy, FM and VALSD were not used). }
\label{table_ablation_2}
\makebox[\linewidth]{
\renewcommand{\arraystretch}{1.2}
\begin{tabular}{cc|ccc}
& & \multicolumn{3}{c}{ Levels} \\
& & 2 & 3 & 4 \\ \hline
\multirow{3}{0.8cm}{Scales} & 1 & \multicolumn{1}{|c}{ 55.24/70.85} & 55.66/70.88 & 55.93/71.05 \\
& 2 & \multicolumn{1}{|c}{55.95/70.77} & 56.19/71.13 & 56.63/71.72 \\
& 3 & \multicolumn{1}{|c}{56.67/71.81} & 57.03/71.98 & \bf 57.28/72.01 \\
\end{tabular}}
\end{table}

\begin{table}[t]
\centering
\caption{\small  Ablation study on {\em mini}-ImageNet \wrt the choice of matching algorithm and mode on the \textit{mini}-Imagenet dataset (1- and 5-shot accuracy). We used the `Conv-4-64' backbone. }
\label{table_ablation_1}
\makebox[\linewidth]{
\setlength{\tabcolsep}{0.6em}
\renewcommand{\arraystretch}{1.2}
\begin{tabular}{c|cccc|cccc}
& \multicolumn{4}{c}{ Intra-level Matching} & \multicolumn{4}{c}{ Inter-level Matching}\\
 shot &  CM &  GM &  OT &  GR &  CM &  GM &  OT &  GR \\ \hline
 1 &  57.49 &  57.83 &  \bf 58.03 &  57.79 &  56.87 &  57.72 &  57.91 &  56.91 \\ 
 5 &  72.21 &  72.58 &  \bf 73.06 &  72.49 &  71.59 &  72.43 &  72.36 &  71.77 \\ 
\end{tabular}}
\end{table}

For the unsupervised pipelines, the U-SoSN and U-MlSo models achieve very promising results on  Open MIC considering no class-wise annotations were used during the training step. To demonstrate this point further, the performance of U-SoSN is better than that of the supervised Relation Net. { The U-MlSo model with the Gate Module further outperforms U-SoSN by up to $\sim$3\% accuracy, which is close to the performance of the supervised SoSN model. Specifically, U-MlSo outperformed supervised SoSN on 4 splits.} The above experiment shows that the unsupervised framework is especially effective in the scenario if (i) the number of training samples from the base classes is limited, and (ii) the problem is closely related to the image retrieval rather than the pure object category recognition. 

\vspace{0.05cm}
\noindent\textbf{\hg Fine-grained FSL.}
{\hg For the fine-grained datasets, our proposed models are evaluated on the CUB Birds \cite{WahCUB_200_2011}, Stanford Dogs \cite{dogs} and Cars \cite{cars} datasets given the 5-way 1-shot and  5-way 5-shot protocols. We follow the same training, validation and testing splits as provided in \cite{huang2020low}. Table \ref{table5} demonstrates that our models outperform the baseline models such as Relation Net \cite{sung2017learning}, LRPABN \cite{huang2020low} and MattML \cite{zhu2020multi}. For CUB Birds, our best MlSo with FM (OT) achieves $\sim\!9.2\%$ and $\sim\!11.0\%$ improvements in accuracy (the 1- and 5-shot protocols, respectively) compared to the 1-st order Relation Net. 
For the Stanford Dogs and Cars datasets, the overall improvements of our MlSo are more significant. To illustrate this point, MlSo outperforms Relation Net by up to $\sim\!12.4\%$ and $\sim\!16.8\%$ accuracy on Stanford Dogs, and $\sim\!20.0\%$ and $\sim\!24.8\%$ on Stanford Cars (the 1- and 5-shot protocols, respectively). Furthermore, our unsupervised FSL variants, U-SoSN and U-MlSo, attain even higher gains in accuracy which are often in the range of between 10\% and 25\% compared to U-Relation Net.}

\vspace{0.05cm}
\noindent\textbf{\hg Few-shot Action Recognition (FSAR).}
{\hg We conclude our evaluations on FSL Action Recognition, denoted as FSAR for short. To this end, results are obtained on the HMDB51 \cite{Kuehne11}, UCF101 \cite{soomro2012ucf101} and \textit{mini}-MIT \cite{monfortmoments} datasets. We follow the  evaluation protocols in \cite{arn}. As every action clip may contains hundreds of frames, we resize all frames to $84\times84$ pixels and downsample clips along the temporal mode to reduce the usage of the GPU memory and limit the computational footprint. Table \ref{table6} shows that our MlSo outperforms the C3D Relation Net, SoSN and CMN \cite{cmn} models. These results are consistent with our evaluations on the category
recognition and fine-grained datasets.}

\subsection{Ablation Studies}
\label{sec:abl}
Below we conduct ablations studies and provide discussions regarding several components of our pipeline.

\vspace{0.05cm}
\noindent\textbf{Relationship Descriptors and Power Normalization.}  
Firstly, Table \ref{table3} shows that the use of PN brings $\sim\!0.8\%$ gain in accuracy over not using it. Not included in the tables, similar were our observations on the Open MIC dataset. Given the simplicity of PN, we have included it in our evaluations unless stated otherwise. 

Relationship descriptors from Table \ref{tab_methods} are evaluated in Table \ref{table3} according to which the single-level SoSN($\otimes$) model outperforms the single-level SoSN( $\otimes$+F) and SoSN($\otimes$+R) models. We expect that averaging over $Z$ support descriptors from $Z$ support images, as in SoSN($\otimes$),  removes the uncertainty in few-shot statistics, whereas the outer-products of support/query datapoints still enjoy the benefit of spatially-wise large convolutional feature maps (which helps form the robust second-order statistics). 

\vspace{0.05cm}
{ \noindent\textbf{Importance of Second-order Pooling}. Table \ref{table_tiered} ({\em tiered}-ImageNet) shows that Relation Net achieves  54.48\% and 71.32 \% accuracy on the 1- and 5-shot protocols, respectively. Relation Net can be considered identical with the SoSN model but it uses first-order pooling. In contrast, SoSN achieves 58.62 and 75.19 \% accuracy (1- and 5-shot protocols, respectively). The gain in accuracy can be attributed to the fact that SoSN is equipped with second-order pooling. 

Table \ref{table3} shows that on {\em mini}-ImageNet, Relation Net scores 50.44\% and 65.32\% accuracy on the 1- and 5-shot protocols, whereas SoSN scores  52.96\% and 68.63\% accuracy, respectively.

The similar trend can be observed on {\em mini}-ImageNet when aggregating over feature maps of images of $256\!\times\!256$ pixels. In such a  setting (not included in the tables as $256\!\times\!256$ pixels resolution is not a part of standard FSL protocol), Relation Net yields 54.01\% and 68.56\%  accuracy on the 1- and 5-shot protocols. In contrast, SoSN yields 57.74\% and 71.08\% accuracy, respectively.

SoP represents features of each image as second-order statistics which are invariant to the spatial order of features in feature maps, and the spatial size of these feature maps. Second-order statistics are also richer than first-order statistics, as indicated by gains in accuracy attained by SoSN in comparison with Relation Net. SoP is the most beneficial when feature representations provide many feature vectors for aggregation \eg, $N\!\geq\!d$. In such a case, the autocorrelation matrices may be of full rank (rank-$d$), thus capturing more statistical information in comparison to the first-order prototypes (equivalent of the rank-$1$ statistic).
}

\vspace{0.05cm}
{ \noindent\textbf{MlSo, SB, FM and VALSD. } Table \ref{table3} demonstrates on \textit{mini}-ImageNet that our MlSo outperforms SoSN and a larger number of prior works. The performance on the 5-way 1-shot and  5-way 5-shot protocols achieves the peak gain of $\sim\!6.0\%$ in accuracy  given 4 encoding levels and the Optimal Transport matching step, compared to the best single-level SoSN model.

Table \ref{table_ablation_3} shows that adding our Scale-wise Boosting (SB), the Feature Matching (FM) and the Visual Abstraction Level and Scale Discriminator (VALSD) yields  around $2\%$ gain in accuracy over the baseline model. For the SB+FM case, we mean that the advanced abstraction level and spatial scale matching strategy is used, thus the loss in Eq. \eqref{eq:sosn_new1} is used for SB+FM, whereas the loss in Eq. \eqref{eq:sosn22} is used for SB alone.

Table  \ref{table_ablation_2} shows that using four level of visual abstraction and three spatial scales is the best, which is consistent with our claim that the levels of visual abstraction and spatial scales need to be taken into account in few-shot learning.

Table \ref{table_ablation_1} shows that the Intra-level Matching strategy is overall better than the Inter-level Matching strategy. This is consistent with our expectations as the levels of visual abstraction and spatial scales are quite complementary. Therefore, matching spatial scales irrespective of abstraction levels is a meaningful strategy. Finally, the Gate Module and the Optimal Transport are two best performing strategies,  followed by the GRaph matching (GR) strategy and the Cosine Matching (CM) strategy. We suspect that GR was somewhat suboptimal due to the small-size dense adjacency matrix in our problem rather than the large scale sparse adjacency matrix which would normally capture a complex topology of some large graph (node classification, \etc).  OT performed robust matching as it is designed to find an optimal transportation plan between different levels of abstraction and spatial scales of support-query pairs. Nonetheless, GM appears to provide the best matching trade-off, that is, GM is almost as good as OT in terms of accuracy, and it is faster than OT in terms of computational complexity (\ie, there is no need to solve any linear programs). 
%
Finally, Figure \ref{fig:cam} shows
the visualization of multiple levels of feature abstraction in our FEN. Across all visualizations and their Gate Module scores (GM), the coarse-to-fine levels of feature abstraction appear to be complementary with each other. 
} 

\ifdefined\arxiv
\newcommand{\PowH}{3.0cm}
\newcommand{\PowHB}{2.875cm}
\newcommand{\PowW}{3.65cm}
\else
\newcommand{\PowH}{3.2cm}
\newcommand{\PowHB}{3.4cm}
\newcommand{\PowW}{3cm}
\fi

\comment{
\begin{figure*}[t]
\begin{subfigure}[t]{0.33\linewidth}
\centering\includegraphics[trim=0 0 0 0, clip=true, height=\PowH]{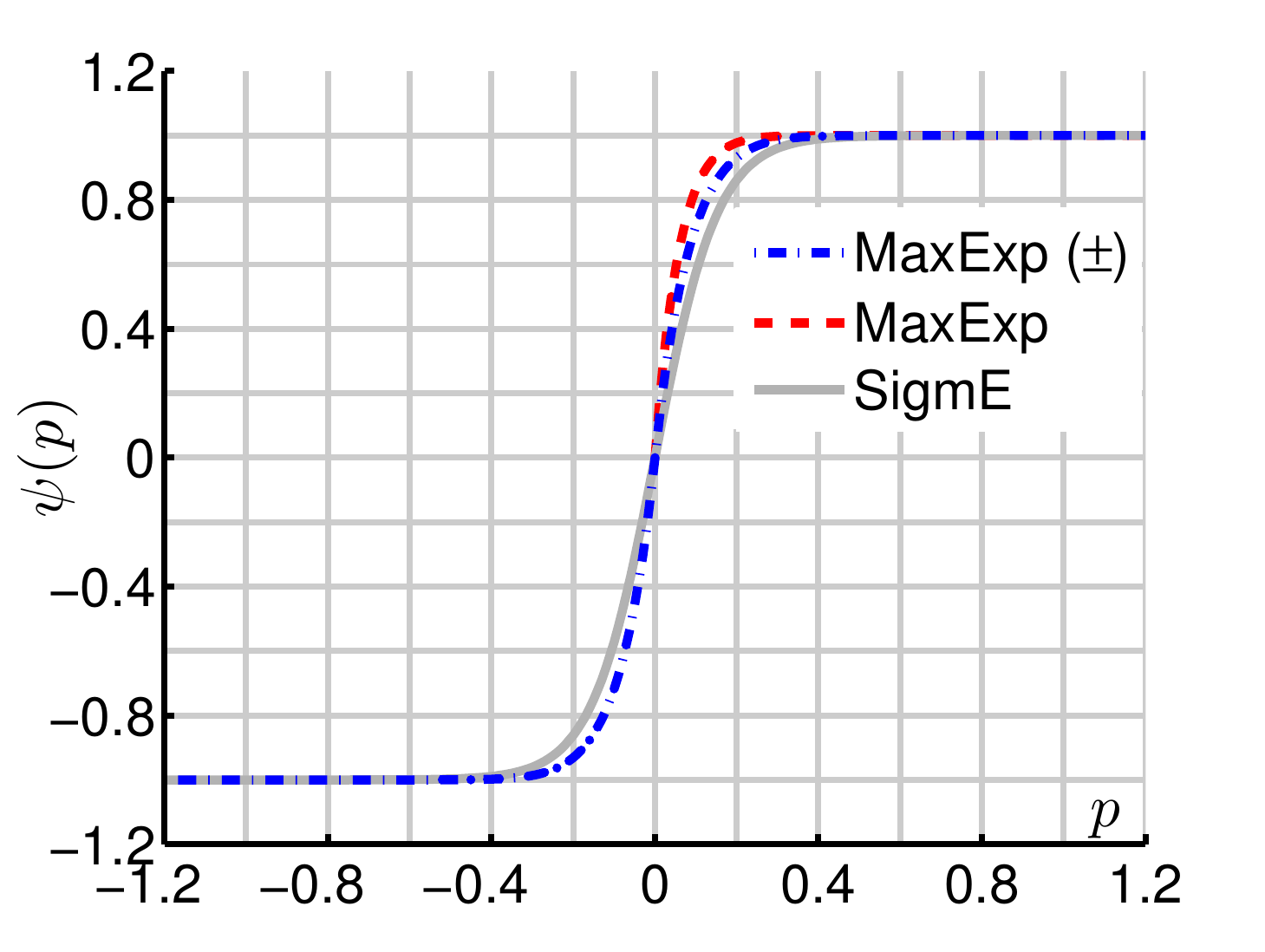}\vspace{-0.2cm}
\caption{\label{fig:pow1}}
\end{subfigure}
\begin{subfigure}[t]{0.33\linewidth}
\centering\includegraphics[trim=0 0 0 0, clip=true, height=\PowH]{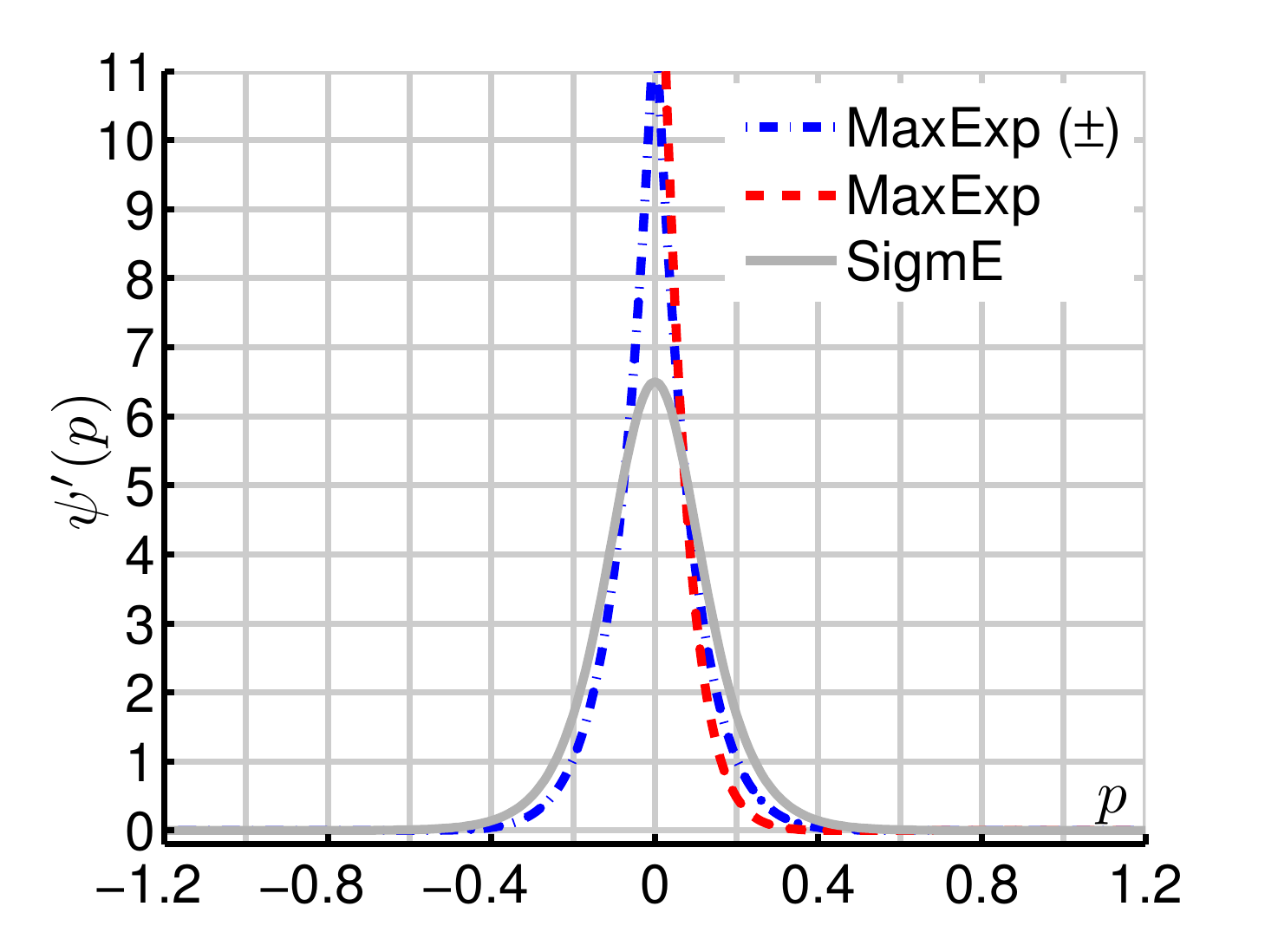}\vspace{-0.2cm}
\caption{\label{fig:pow2}}
\end{subfigure}
\begin{subfigure}[t]{0.33\linewidth}
\centering\includegraphics[trim=0 0 0 0, clip=true, height=\PowH]{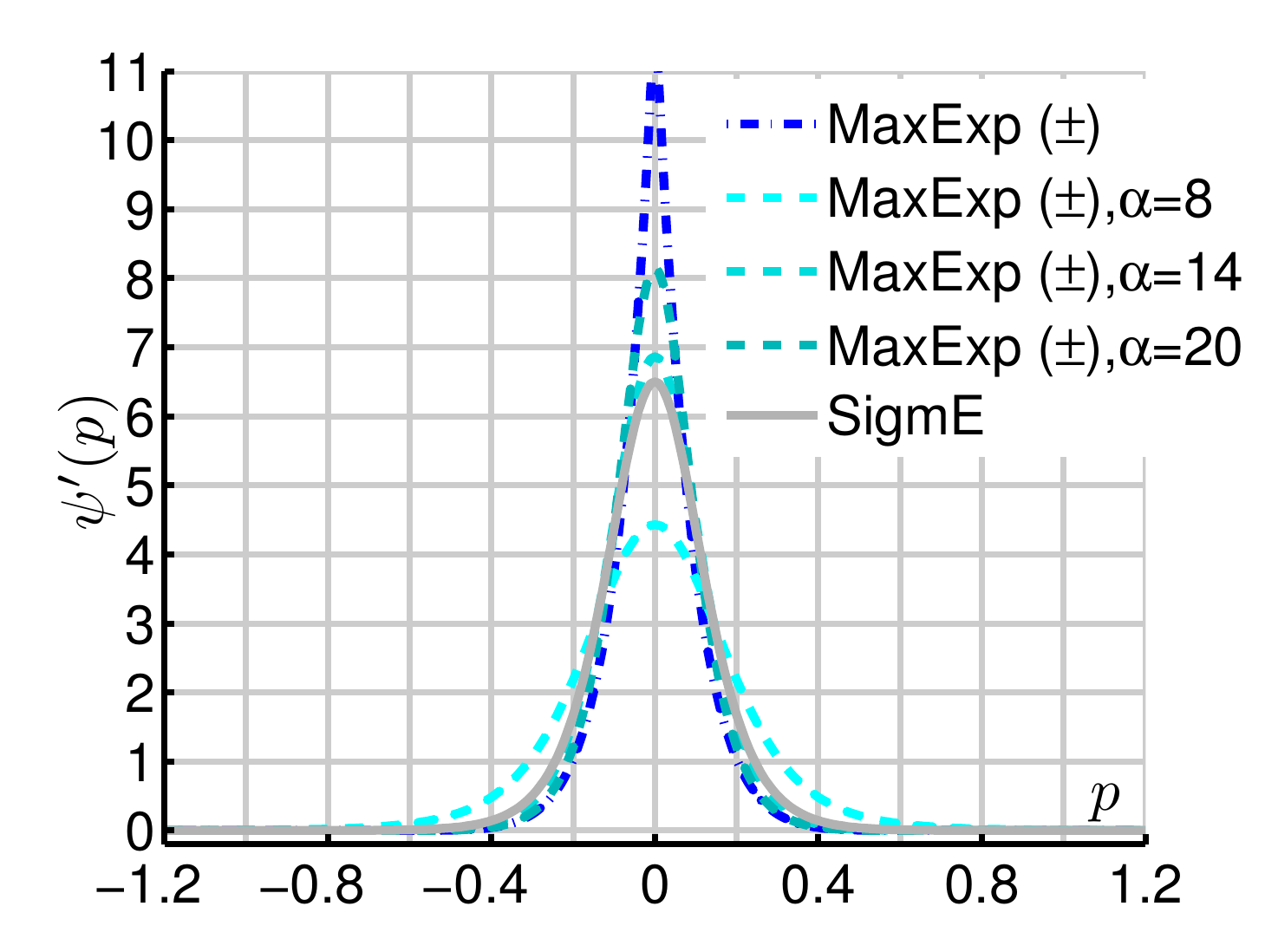}\vspace{-0.2cm}
\caption{\label{fig:pow3}}
\end{subfigure}
%


\caption{MaxExp$\,$($\pm$), MaxExp, and SigmE are in Figure \ref{fig:pow1}. Note that MaxExp$\,$($\pm$) extends MaxExp to negative values and is closely approximated by SigmE. Figure \ref{fig:pow2}: Derivatives of MaxExp$\,$($\pm$), MaxExp, and SigmE. Figure \ref{fig:pow3}: Smooth derivative of MaxExp$\,$($\pm$) with the soft maximum.}
\label{fig:power-norms}
\end{figure*}
}

\section{Conclusions}
\label{sec:conclusions}

We have presented the end-to-end trainable SoSN and MlSo models for supervised and unsupervised few-shot learning. With the use of  multiple levels of feature abstractions, multiple spatial scales, the Gate Module and the Visual Abstraction Level and Scale Discriminator, we have shown how to  learn efficiently the similarity between the support-query pairs captured by the relation descriptors.  We have investigated how to capture relations between the query and support matrices, and how to produce multiple complementary levels of feature information. We have also investigated several strategies for the level of abstraction and spatial scale matching to showcase the importance of decomposing the support-query pairs into multiple spatial scales at multiple levels of feature abstraction. 
MlSo demonstrates consistent large gains in accuracy across all benchmarks for both supervised and unsupervised learning. 
Given the simplicity of our approach, we believe that the SoSN and MlSo models are interesting propositions that can serve as a starting point in designing more elaborate FSL approaches. 

\ifCLASSOPTIONcaptionsoff
  \newpage
\fi

{\small
\bibliographystyle{ieee_fullname}
\bibliography{TMM}
}

\begin{IEEEbiography}
[{\includegraphics[width=1in,height=1.25in,clip,keepaspectratio]{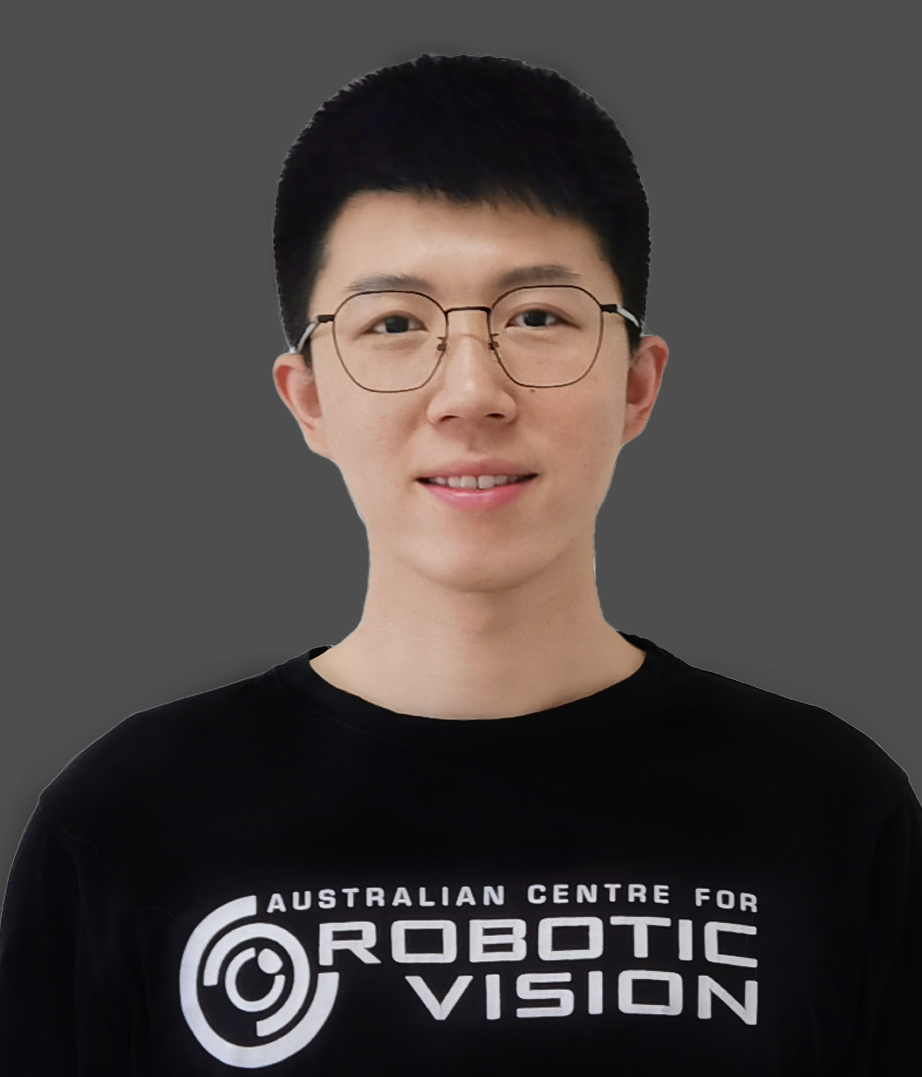}}]{Hongguang Zhang}
Hongguang Zhang is currently an assistant professor in Systems Engineering Institute, AMS. Before this, he received his PhD degree in computer vision and machine learning at the Australian National University and Data61/CSIRO, Canberra, Australia in 2020. He received the BSc degree in electrical engineering and automation from Shanghai Jiao Tong University, Shanghai, China in 2014. He received his MSc degree in electronic science and technology from National University of Defense Technology, Changsha, China in 2016. His interests include fine-grained image classification, zero-shot learning, few-shot learning and deep learning methods.
\end{IEEEbiography}
\vspace{-1cm}
\begin{IEEEbiography}
[{\includegraphics[width=1in,height=1.25in,clip,keepaspectratio]{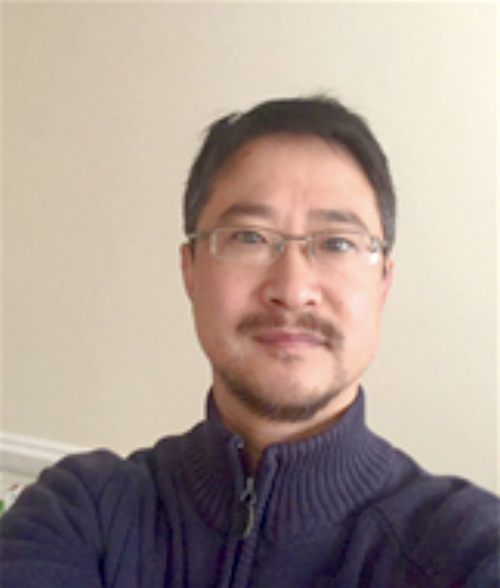}}]{Hongdong Li} was with NICTA Canberra Labs, prior
to 2010, where he involved in the “Australia Bionic
Eyes” Project. He is currently a Professor with the
Computer Vision Group, The Australian National
University. He is also a Chief Investigator of the
Australia ARC Centre of Excellence for Robotic
Vision. His research interests include 3D vision
reconstruction, structure from motion, multi-view
geometry, as well as applications of optimization
methods in computer vision. He was a recipient of
the CVPR Best Paper Award in 2012 and the Marr
Prize Honorable Mention in 2017. 
\end{IEEEbiography}
\vspace{-1cm}
\begin{IEEEbiography}
[{\includegraphics[width=1in,height=1.25in,clip,keepaspectratio]{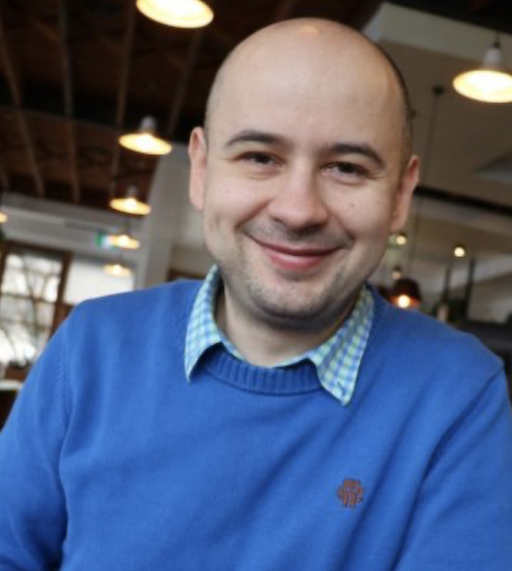}}]{Piotr Koniusz}
A Senior Research Scientist in Machine
Learning Research Group at Data61/CSIRO,
formerly known as NICTA, and a Senior Honorary Lecturer at Australian National University
(ANU). Previously, he worked as a postdoctoral
researcher in the team LEAR, INRIA, France.
He received his BSc degree in Telecommunications and Software Engineering in 2004 from the
Warsaw University of Technology, Poland, and
completed his PhD degree in Computer Vision
in 2013 at CVSSP, University of Surrey, UK. His
interests include visual categorization, spectral learning on graphs and
tensors, kernel methods and deep learning.
\end{IEEEbiography}





\end{document}